\documentclass{article}

\usepackage{balance} 

\usepackage[margin=3cm]{geometry}

\usepackage{amsmath,amssymb,amsthm,mathtools}
\usepackage[utf8]{inputenc}
\usepackage{xspace}
\usepackage{url}
\usepackage{dsfont}
\usepackage{enumerate}
\usepackage{rotating}
\usepackage{lscape}
\usepackage{hyperref}
\usepackage{booktabs}

\usepackage{caption}
\usepackage{subcaption}

\usepackage{balance}
\usepackage[shortcuts]{extdash}

\usepackage{enumitem}

\usepackage[numbers,longnamesfirst]{natbib}

 \newtheorem{theorem}{Theorem}
 \newtheorem{definition}[theorem]{Definition}
 
 \newtheorem{lemma}[theorem]{Lemma}
 \newtheorem{corollary}[theorem]{Corollary}
 \newtheorem{claim}[theorem]{Claim}
 
 \newtheorem{situation}[theorem]{Situation}
 \numberwithin{theorem}{section}

\usepackage{tikz}
\usetikzlibrary{shapes,arrows}
\usetikzlibrary{decorations}
\usetikzlibrary{plotmarks}
\usetikzlibrary{calc}
\usetikzlibrary{mindmap}
\usetikzlibrary{shadows}
\usetikzlibrary{backgrounds}
\usetikzlibrary{shapes}
\usetikzlibrary{shapes.symbols}

\usepackage{pgfplots}

\usepackage[algo2e,ruled,vlined,linesnumbered]{algorithm2e}
\SetAlgoSkip{}
\DontPrintSemicolon

\allowdisplaybreaks[3]

\theoremstyle{remark}

\newenvironment{proofofclaim}{\textsc{Proof of Claim.}}{\hfill\scriptsize\scalebox{0.75}{$\blacksquare$}}

\newcommand*{\om}{\textsc{OneMax}\xspace}

\newcommand*{\onemax}{\om}

\newcommand*{\jump}{\textsc{Jump}\xspace}

\newcommand{\hurdle}{\textsc{Hurdle}\xspace}

\DeclareMathOperator{\Prob}{Pr}

\newcommand*{\E}{\mathrm{E}}

\DeclareMathOperator{\Bin}{Bin}

\newcommand{\N}{\mathds{N}}

\newcommand{\eps}{\varepsilon}

\newcommand{\plateau}{\textsc{plateau}\xspace}

\newcommand{\ones}[1]{\left|#1\right|_1\xspace}
\newcommand{\zeros}[1]{\left|#1\right|_0\xspace}

\newcommand{\pr}{\ensuremath{\mathrm{Pr}}}

\newcommand{\EA}{(1+1)~EA\xspace}
\newcommand{\muea}{($\mu$+1)~EA\xspace}
\newcommand{\muga}{($\mu$+1)~GA\xspace}
\newcommand{\mulga}{($\mu$+1)~GA\xspace}
\newcommand{\lambdac}{{\lambda_c}}
\newcommand{\mulgavar}{\texorpdfstring{($\mu$+1)\nobreakdash-$\lambdac$\nobreakdash-GA\xspace}{(mu+1)-lambda-c-GA}}

\newcommand{\mutation}{\mathrm{mutation}}
\newcommand{\mut}{\mathrm{mut}}
\newcommand{\cross}{\mathrm{c}}

\newcommand{\poly}{\mathrm{poly}}
\newcommand{\popt}{p^*}
\newcommand{\rem}{\mathrm{rem}}

\newcommand{\toggleplot}[1]{{\textcolor{red}{Plots removed to increase compilation speed. Use command $\backslash$toggleplot in preamble to reinsert them.}}}

\newcommand{\newedit}[1]{\textcolor{black}{#1}}
\newcommand{\neweditx}[1]{\textcolor{black}{#1}}
\newcommand{\newedity}[1]{{#1}}

\author{
  Andre Opris\\
  University of Passau\\
  Passau, Germany
  \and Johannes Lengler\\
  ETH Z\"urich\\
  Z\"urich, Switzerland
  \and Dirk Sudholt\\
  University of Passau\\
  Passau, Germany
  }

\title{Achieving Tight $O(4^k)$ Runtime Bounds on Jump$_k$ by Proving that Genetic Algorithms Evolve Near-Maximal Population Diversity}

\begin{document}

\maketitle

\begin{abstract}
The $\jump_k$ benchmark was the first problem for which crossover was proven to give a speed-up over mutation-only evolutionary algorithms. Jansen and Wegener (2002) proved an upper bound of $O(\poly(n) + 4^k/p_c)$ for the ($\mu$+1)~Genetic Algorithm (\muga), but only for unrealistically small crossover probabilities~$p_c$. To this date, it remains an open problem to prove similar upper bounds for realistic~$p_c$; the best known runtime bound\neweditx{, in terms of function evaluations,} for $p_c = \Omega(1)$ is $O((n/\chi)^{k-1})$, $\chi$ a positive constant.
We provide a novel approach and analyse the evolution of the population diversity, measured as sum of pairwise Hamming distances, for a variant of the \muga on $\jump_k$. 
The \mulgavar creates one offspring in each generation either by applying mutation to one parent or by applying crossover $\lambdac$ times to the same two parents (followed by mutation), to amplify the probability of creating an accepted offspring in generations with crossover.
We show that population diversity in the \mulgavar converges to an equilibrium of near-perfect diversity. This yields an improved time bound of $O(\mu n \log(\mu) + 4^k)$ \neweditx{function evaluations} for a range of~$k$ under the mild assumptions $p_c = O(1/k)$ and $\mu \in \Omega(kn)$. For all constant~$k$\neweditx{,} the restriction is satisfied for some $p_c = \Omega(1)$ and it implies that the expected runtime for all constant~$k$ \neweditx{and an appropriate $\mu = \Theta(kn)$} is bounded by $O(n^2 \log n)$, irrespective of~$k$. For larger~$k$\neweditx{,} the expected time of the \mulgavar is $\Theta(4^k)$, which is tight for a large class of unbiased black-box algorithms and faster than the original \muga by a factor of $\Omega(1/p_c)$. We also show that our analysis can be extended to other unitation functions such as $\jump_{k, \delta}$ and \hurdle.
\end{abstract}



\textbf{Keywords:} Runtime analysis, diversity, population dynamics

\section{Introduction and Motivation}

        Evolutionary algorithms (EAs) are randomised search heuristics inspired by principles of natural evolution. They evolve a population (a multiset) of search points to explore different parts of the search space in parallel. In one generation, they apply evolutionary operators such as mutation, crossover (also called recombination) and selection to create new search points from the current population, and to create a new population from the most promising search points. 
        
        Evolutionary algorithms are known for their ability to generate novel and unexpected solutions---a hallmark of creativity~\cite{Lehman2023}---and have been described as the ``next deep learning''~\cite{Miikkulainen2021} and the future of creative AI~\cite{VenturebeatArticle}.
        EAs are applicable in a black-box setting where knowledge about the problem in hand is not available or difficult to model. Their population-based approach implies that EAs search the solution space more broadly than many traditional optimisation methods, which increases the likelihood of discovering a global optimum. 

        A major drawback is that their dynamic behaviour is as little understood as it was almost two decades ago~\cite{HornbyYuECSurvey}. It is often not clear when and why EAs perform well, and which components contribute most to their success~\cite{BookNeuWit,Jansen2013,doerr-neumann-book}. Even some of the most fundamental questions are unresolved. 
        Since the invention of evolutionary algorithms (EAs), researchers have wondered about the usefulness of the crossover operator. For decades, only empirical evidence was available\newedit{,} suggesting that using crossover in addition to mutation improves performance.

        In a breakthrough result, \citet{Jansen2002} presented the first proof that the use of crossover can speed up the expected optimisation time of a \muga, for the $\jump_k$ \newedit{benchmark} (defined in Section~\ref{sec:preliminaries}) in which a ``jump'' of changing $k$ bits is typically required to reach a global optimum. However, their result had a severe limitation: 
        the crossover probability had to be chosen as $p_c = O(1/(kn))$, which for large problem sizes $n$ is unrealistically small. Many researchers \newedit{already succeeded} to improve on these results~\cite{Koetzing2011a,Dang2017,Doerr2024} (see Table~\ref{tab:overview-runtime-results} and related work below), and  $\jump_k$ is by now the most intensely studied theoretical benchmark for the benefits of crossover. However, despite intensive research, the basic questions remained wide open. Excluding the addition of aggressive diversity-enhancing mechanisms~\cite{Dang2016,Ren2024} or voting mechanisms~\cite{Friedrich2016,Whitley2018,Rowe2019}, the results fall into two categories: one line of work could show strong effects of crossover, but as in~\cite{Jansen2002} one has to restrict to unrealistically small $p_c$ that go to zero with $n$. For constant $k$\neweditx{,} no better asymptotic than $p_c = O(1/n)$ could be achieved in this case. A second line of research made less artificial assumptions on~$p_c$, but could only show very moderate benefits from crossover. The best known runtime bound in this range was $O((n/\chi)^{k-1})$~\cite{Doerr2024}, where $\chi/n$ is the mutation \neweditx{probability}. For the default choice $\chi =1$ and constant $k$, this saves a moderate factor of $n$ compared to the runtime of $O(n^k)$ for mutation-only algorithms. All this was in strong contrast to
        empirical evidence~\cite{Dang2017,Li2023}, which suggests that the expected optimisation time is much smaller than $O((n/\chi)^{k-1})$, also for large~$p_c$.

        \subsection{Our contribution} \newedit{We give massively improved bounds} in the regime of large $p_c$. 
        We study a variant of the \muga that we call \emph{\mulga with competing crossover offspring}, denoted as \mulgavar for short, under a mild condition on $p_c$. In a generation executing crossover, the \mulgavar creates $\lambdac$ offspring by crossing the same parents, applying mutation and then continuing the replacement selection with the best such offspring. In generations without crossover, it creates one offspring as usual. 
        Hence, the \mulgavar remains a steady-state algorithm, creating one new offspring in each generation.
        The reasons for creating multiple offspring in generations with crossover are as follows. When the whole population consists of local optima of $\jump_k$, which are search points with $n-k$ bits set to~1, crossing two local optima with a large Hamming distance is likely to create an individual of worse fitness that would be removed immediately. If the two parents have small Hamming distance then the offspring is more likely to be accepted. Hence, among local optima of $\jump_k$, pairs of small Hamming distance are more likely to generate accepted offspring than pairs of large Hamming distance. The \mulgavar evens out this imbalance: by creating several offspring from the same pair of parents, it amplifies the probability of creating at least one acceptable offspring per generation. Our choice of $\lambdac$ will ensure that all pairs of parents reach roughly the same probability of almost $1$. Hence, the \mulgavar prevents idle generations with crossover. 
        
        A second effect of the \mulgavar is that it amplifies the probability of creating the global optimum in a generation executing crossover. This makes sense since, as we will show, the algorithm evolves sufficient diversity such that generations with crossover are much more likely to create the optimum than those using only mutation. Steps without crossover are still needed to evolve and maintain a sufficient diversity. But to achieve this, a single offspring is sufficient.

        Our main results are as follows. 
        We show an upper bound \neweditx{on the expected number of function evaluations} of 
        \begin{equation}
            \label{eq:final-bound-in-intro}
            O \big(\mu n\log(\mu) + 4^k \big)
        \end{equation}
        when $k = O(\sqrt{n})$ and for suitable $\lambdac = \Omega(\sqrt{k} \neweditx{\; \log}(\mu) + 1/p_c)$, $p_c = O(1/k)$ and polynomial $\mu = \Omega(kn)$ \neweditx{(see Theorem~\ref{the:Runtime-2})}. When $k \ge \log(\mu n \log(\mu))/2$ and $k = O(\sqrt{n})$,  this is dominated by the term $4^k$, \newedit{transferring the upper bound $O(4^k/p_c)$ by~\citet{Jansen2002} to large $p_c$ and improving it by a factor of $1/p_c$}. Our analysis applies to crossover probabilities that are larger by a factor of~$n$ compared to~\cite{Jansen2002} and it holds for a larger range of gap lengths, $k= O(\sqrt{n})$ as opposed to $k = O(\log n)$ from~\cite{Jansen2002}. Another interesting regime is $k \le \log(\mu n \log(\mu))/2$, 
        where we obtain an upper bound $O(\mu n \log(\mu))$. \neweditx{When using $\mu = \Theta(nk)$, this bound simplifies to $O(kn^2 \log(kn)) \subseteq O(n^2\log^2 n)$, and for constant~$k$, we obtain $O(n^2\log n)$.} In both cases, the expected optimisation time can be bounded by a polynomial that does not depend on~$k$. 
        This is a massive improvement over the previously mentioned best bound of $O(n^{k-1})$ for the original \muga~\cite{Dang2017,Doerr2024} in this range.

        In addition, while working towards proving~\eqref{eq:final-bound-in-intro} we prove a more general upper bound of
        \begin{equation}
            \label{eq:general-bound-in-intro}
            O \left(\lceil\lambdac p_c\rceil \mu(n\log(1/\eps) + \log(\mu)) + \lambdac + \frac{1}{p_c} + 4^{k-\lfloor 8\eps k\rfloor}\left(n/\chi\right)^{\lfloor8\eps k\rfloor} \left(1 + \frac{1}{p_c\lambdac}\right) \right)
        \end{equation}
        \neweditx{expected function evaluations} that makes much fewer and looser assumptions on parameters (see Theorem~\ref{the:Runtime-1}). The bound holds for all $\eps >0$, but we require $p_c = O(\eps)$ with a suitable hidden constant. Loosely speaking, the optimal choice of $\eps$ indicates how close population diversity at the equilibrium state is to the maximum possible diversity. The simpler bound~\eqref{eq:final-bound-in-intro} chooses $\eps = \eps(k) = \Theta(1/k)$ such that $\lfloor 8\eps k\rfloor = 0$ and the term $(n/\chi)^{\lfloor 8 \eps k\rfloor}$ vanishes. However, since we require $p_c = O(\eps)$, this limits crossover probabilities to $p_c = O(1/k)$ in~\eqref{eq:final-bound-in-intro}. The general bound above~\eqref{eq:general-bound-in-intro} also holds for (small) constant values of $p_c$ at the expense of a generally looser bound than~\eqref{eq:final-bound-in-intro}. In this case, $\eps$ can be chosen as a small constant value, and the factor $(n/\chi)^{\lfloor8\eps k\rfloor}$ becomes dominant. This is still a significant speed-up compared to the best known $O((n/\chi)^{k-1})$ bound, showing that constant crossover probabilities can already yield drastic speed-ups.
        

        Our analysis reveals insights into the evolutionary dynamics of the population diversity. While previous analyses for large $p_c$ could only show that crossover happens between parents of Hamming distance at least $1$ (see below for details), our results for $\eps = \Theta(1/k)$ prove that a constant fraction of all crossovers happens between parents of Hamming distance $2k$, which is the largest possible Hamming distance \newedit{on the set of local optima of} $\jump_k$, \newedit{which we call \plateau}. Our improved analyses build on \newedit{our} previous {work~\cite{Lengler2024}}, \newedit{which} studied equilibrium states for population diversity on a flat (i\neweditx{.}e.\newedit{,}\ constant) fitness function. Here we translate \newedit{this} approach to \newedit{the set \plateau}. Diversity is measured as the pairwise Hamming distance between search points in the population, and we show that this measure is well\newedit{-}suited to capture the dynamic behaviour of the \mulgavar. 

        \subsection{Related work on crossover for $\jump_k$}
        \citet{Jansen2002} presented the first rigorous theoretical proof that crossover can speed up optimisation. On the function $\jump_k$ (formally defined in Section~\ref{sec:preliminaries})\neweditx{,} evolutionary algorithms typically get stuck on \newedit{\plateau} and then have to cross a fitness valley to ``jump'' to the optimum $1^n$ by flipping \newedit{the remaining} $k$ zeros to one. For 
        $k = O(\log n)$\newedit{,} they proved that a ($\mu$+1)~Genetic Algorithm (\muga) using crossover optimises the function in expected polynomial time, whereas the \EA using only standard bit mutation requires expected time $\Theta(n^k)$.
        This holds for the standard \muga and a variant that prevents replicates from entering the population in steps where no crossover is executed.
        The reason for crossover's superior performance is that, once the population only contains individuals with $n-k$ ones, crossover can combine parents having zeros at different positions to create the global optimum $1^n$. In the best case, two parents have zeros at disjoint positions. This is called a \emph{complementary pair} and then the probability of uniform crossover turning all these $2k$ positions to~1 is $(1/2)^{2k} = 4^{-k}$. 

        \Citet{Jansen2002} showed that the \muga typically evolves a population in which at any fixed bit position at most $\mu/(2k)$ individuals carry a zero. Then \newedit{for any fixed parent~$x$\neweditx{,} and any of its \neweditx{zero}-bits, there are at most $\mu/(2k)$ other search points sharing a zero at this position, and there are at most $k \cdot \mu/(2k) = \neweditx{\mu/2}$ search points sharing any \neweditx{zero}-bit with~$x$. Thus,} the probability of selecting a complementary pair as parents is at least $1/2$. 
        This analysis, however, comes with two limitations. It only works for small~$k$, $k = O(\log n)$. More importantly, it requires the crossover probability $p_c$ to be unrealistically small: $p_c = O(1/(kn))$. The reason is that it was unclear how crossover affects the diversity of the population and this was handled using pessimistic \newedit{arguments}.
        It was left as an open problem to analyse the \muga with more realistic crossover probabilities.
        
        The results were later refined by~\citet{Koetzing2011a} using a different analysis for the \muga avoiding replicates in mutation steps. By showing a sequence of events that suffice to evolve one complementary pair, they showed a time bound\footnote{The original proof in~\cite{Koetzing2011a} only works for $p_c = k/n$ as smaller $p_c$ are not considered when estimating the probability of the event E6. We expand on this and show how to fix this issue in the appendix.} of $O(\mu n \log(n) + e^{6k}  \mu^{k+2}  k/p_c)$ for all $p_c \le k/n$.
        Efficient performance could be shown for a larger range of~$k$, $k = o(\sqrt{n})$ and $p_c = O(k/n)$. However, this \newedit{$p_c$} is still unrealistically small as $n$ becomes large.
        
        \Citet{Dang2016} showed that, when equipping the \muga with explicit diversity-enhancing mechanisms such as fitness sharing, expected times of $O(n \log(n) + \mu^2 k n \log(\mu k) + 4^k/p_c)$ can be achieved. A similar result was presented recently in~\cite{Ren2024}.
        Since only the dynamics on the plateau of search points with $n-k$ ones matters, it suffices to use these diversity mechanisms in the \muga{}'s tie-breaking rule. However, this does not answer the question why the standard \muga is efficient.
        
        \Citet{Dang2017} were the first to upper bound the expected time for the standard \muga with crossover probability $p_c=1$, that is, always using crossover. They measured diversity in terms of the size of the largest species, where a species consists of copies of the same genotype. They showed that the \muga efficiently evolves and keeps a good amount of diversity, such that there is a probability of $\Omega(1)$ of selecting two parents of different genotypes. This implies that not all bits must be flipped by mutation\neweditx{,} since some bits can be more easily set to~1 by a uniform crossover. They showed a time bound of $O(n^{k-1} \log n)$ for the best possible population size, which improves to $O(n^{k-1})$ when increasing the mutation probability from $1/n$ to $\chi/n$ for a constant $\chi > 1$. 
        
        Very recently, \citet{Doerr2024} improved the analysis from~\cite{Dang2017} and showed that diversity, in terms of the largest species, remains for much longer, for a time exponential in the population size. This yields an improved bound of $O((n/\chi)^{k-1})$ for constant $\chi$ and $p_c = \Omega(1)$ and the best possible choice of the population size~$\mu$.
        An overview of all previous runtime bounds\footnote{We simplify the bound stated in~\citet{Doerr2024} using simple algebraic arguments. Details are shown in the appendix.} (excluding those using aggressive diversity mechanisms) is shown in Table~\ref{tab:overview-runtime-results}.

\begin{sidewaystable}
\caption{Overview of runtime bounds for the \muga and \mulgavar and restrictions on parameters. Here $\varepsilon > 0$ is a small enough positive number\neweditx{, $\log$ denotes the base-2 logarithm to base $2$, and $\ln$ denotes the natural logarithm}. Standard bit mutation with mutation probability $\chi/n$ is used, with a default of $\chi=1$.
Some algorithms exclude mutations where no bit flips (``replicates''), in steps without crossover. Upper bounds from~\cite{Koetzing2011a,Doerr2024} are adapted from the original source. The lower bounds apply to all unbiased mutation operators, possibly preceded by uniform crossover, from parents from $\plateau$ (see Theorem~\ref{the:lower-bound}). For bounds for the \muga equipped with explicit diversity-enhancing mechanisms we refer to~\cite{Dang2016,Ren2024}}.
\label{tab:overview-runtime-results}
\begin{tabular*}{\textheight}{@{\extracolsep\fill}l@{\ }l@{\ }c@{\ }c@{\ }c}
    algorithm / $\lambdac$ & runtime bound & $p_c$ & $\mu$ & $k$\\
    \toprule
    \muga w/o replicates~\cite{Jansen2002} & $O(\mu n(k^2 + \log(\mu n)) + 4^k/p_c)$ & $O(1/(kn))$ &  $\ge k\log^2 n$ & $O(\log n)$ \\ 
    \muga~\cite{Jansen2002} & $O(\mu n^2k^3 + 4^k/p_c)$ & $O(1/(kn))$ & $\ge k\log^2 n$, $n^{O(1)}$ & $O(\log n)$ \\ 
    \muga w/o replicates~\cite{Koetzing2011a} & $O(\mu n \log(n) + e^{6k} \mu^{k+2}k/p_c)$ & $\le k/n$ & $\ge 2$, $n^{O(1)}$ & $o(\sqrt{n})$\\ 
    \muga~\cite{Dang2017} & $O(\mu n \sqrt{k}\log(\mu) + n^k/\mu + n^{k-1}\log(\mu)))$ & $1$ & $\le \varepsilon n$ & $o(n)$\\
    \muga~\cite{Dang2017}, $\chi = 1 + \Theta(1)$ & $O(\mu n \sqrt{k}\log(\mu) + \mu^2 + n^{k-1})$ & $1$ & $ \ge k\ln(n)/\varepsilon$ & $o(n)$\\
    \muga~\cite{Doerr2024}, $\chi = \Theta(1)$ & $O\left( \mu n \log(\mu) + n \log(n) + \left(\frac{n}{\chi}\right)^{k-1}\left(1 + 
     \frac{(n/k+\mu)\log(\mu)}{\exp(p_c\mu/(2048e))}\right)\right)$ & $\Omega(1)$ & $\ge 2, o(n)$ & $o(n)$\\
    \midrule
    \multicolumn{5}{@{}l}{General bounds for \mulgavar \neweditx{(Theorem~\ref{the:Runtime-1})}:}\\
    $\lambdac \ge 6\sqrt{k}e^{\chi}\ln(\mu)$ & $O \left(\lceil\lambdac p_c\rceil \mu(n\log(1/\eps) + \log(\mu)) + \lambdac + \frac{1}{p_c}\right.$  & $O(\varepsilon)$ & $\Omega(n/\varepsilon)$ & $o(\newedit{n}) \cap O(\varepsilon n)$\\
    & \hspace*{3cm} $\left. + 4^{k-\lfloor 8\eps k\rfloor}\left(n/\chi\right)^{\lfloor8\eps k\rfloor} \left(1 + \frac{1}{p_c\lambdac}\right) \right)$\\
    $\lambdac \ge \max\{6\sqrt{k}e^{\chi}\ln(\mu), 1/p_c\}$ & $O \left(\lceil\lambdac p_c\rceil \mu(n\log(1/\eps) + \log(\mu)) + \lambdac + + 4^{k-\lfloor 8\eps k\rfloor}\left(n/\chi\right)^{\lfloor8\eps k\rfloor} \right)$  & $O(\varepsilon)$ & $\Omega(n/\varepsilon)$ & $o(\newedit{n}) \cap O(\varepsilon n)$\\
    \midrule
    \multicolumn{5}{@{}l}{Bounds for \mulgavar with $\eps = 1/(16k)$ \neweditx{(Theorem~\ref{the:Runtime-2})}:}\\
    $\lambdac \ge 6\sqrt{k}e^{\chi}\ln(\mu)$ & $O \big(\lceil \lambdac p_c\rceil \mu(n\log(k) + \log(\mu)) + \lambdac + 4^k \big)$ & $O(1/k)$ & $\Omega(kn)$ & $O(\sqrt{n})$\\
    $\lambdac \coloneqq \lceil \max\{6\sqrt{k}e^{\chi}\ln(\mu), 1/p_c\} \rceil$ & $O \big(\mu n\log(\mu) + 4^k \big)$ & \hspace*{-1.2cm}$O(1/k) \cap \Omega(4^{-k})$ & $\Omega(kn) \le 2^{n\log k}$ & $O(\sqrt{n})$\\
    \midrule
    \multicolumn{5}{@{}l}{Lower bounds \neweditx{(Theorem~\ref{the:lower-bound})}:}\\
    all unbiased ($\mu$+$\lambda$)~GAs & $\Omega\left(\min\{4^k/p_c, \binom{n}{k}/(1-p_c)\}\right)$ & $\in (0, 1)$ &  & $\le \sqrt{n}/2 $ \\
    all unbiased black-box algorithms & $\Omega\left(\min\{4^k, \binom{n}{k}\}\right)$ & $\in (0, 1)$ &  & $\le \sqrt{n}/2$ \\
    \bottomrule
\end{tabular*}
\end{sidewaystable}

\subsection{Outline and Extension} 
We first \newedit{give} an exact formula for the equilibrium of population diversity for the standard \muea without crossover in Section~\ref{sec:equilibrium-states}. These results are formulated and proven for arbitrary \textit{unbiased} mutation operators\neweditx{, that is, mutation operators that treat all bit values and all bit positions symmetrically~\cite{Lehre2012}}. In Section~\ref{sec:equilibrium-states2} we then bound the equilibrium for the \mulgavar.
Runtime bounds are derived from this in Section~\ref{sec:Runtime-Analysis}, including a lower bound showing that the dominant term $4^k$ is necessary for a large class of unbiased black-box algorithms, and that a stronger lower bound of $\Omega(4^k/p_c)$ applies to algorithms creating the same number of offspring in generations with or without crossover. This lends credence to the design of the \mulgavar and justifies why it creates more offspring in generations with crossover.
Related to this, in Section~\ref{sec:equilibrium-states2} we also give an example of a population for which the drift in terms of population diversity is positive for the \mulgavar, but negative for the original \muga (cf.\ Lemmas~\ref{lem:negative-standard-GA} and~\ref{lem:negative-standard-GA-2}). This counterexample refutes the conjecture that our drift bounds can be transferred to the original \muga. Proofs for these lemmas are placed in an appendix to keep the main part streamlined.
Finally, in Section~\ref{sec:extensions} we show that our analyses can be extended to other functions of unitation, including a generalisation of the $\jump_k$ function and the \hurdle function that requires multiple jumps to reach the global optimum.

This manuscript extends a preliminary conference paper~\cite{Opris2024} in several ways. This paper contains full proofs that had to be omitted in~\cite{Opris2024} owing to space restrictions. We have added Section~\ref{sec:extensions} on applications to other fitness functions and the mentioned counterexample for the drift in the original \muga (Lemmas~\ref{lem:negative-standard-GA} and~\ref{lem:negative-standard-GA-2}). The analyses in Section~\ref{sec:Runtime-Analysis} have been rewritten to increase rigour and to fix some technical issues in~\cite{Opris2024}. In particular, our revised analysis treats rounding issues more rigorously. This avoids problems with pathological parameter settings like excessively large values for $\lambdac$ or $1/p_c$, under which our previous bounds could become incorrect. In the last bound of Theorem~5.2 in~\cite{Opris2024} we previously chose $\lambdac = O(1/p_c)$ to absorb a factor of $\lceil \lambdac p_c\rceil$. However, this introduced the limitation $k = \Omega(\log^2 \mu)$. Here we \neweditx{use} a different argument to bound $\lceil \lambdac p_c\rceil$ that also works for constant~$k$.

Most importantly, we noticed that our bounds could be improved for appropriate parameter settings by replacing the term $O(4^k/p_c)$ with $O(4^k)$. The intuition behind the improved bound is as follows. If $\lambdac \ge 1/p_c$\neweditx{,} then the \mulgavar, on average, spends at least half of its fitness evaluations in generations with crossover. Since we show that the probability of generating the optimum in a constant fraction of those offspring creations is $\Omega(4^{-k})$, the upper bound only contains an additive term of $O(4^k)$. This contradicts the lower bound of $\Omega(4^k/p_c)$ claimed in Theorem~5.3 of~\cite{Opris2024}. Its proof was faulty since it computed the fraction of offspring creations by crossover incorrectly, not taking into account that generations with crossover create several offspring. This is fixed in the present manuscript: we give a general lower bound of $\min\{\neweditx{\Omega}(4^k), \binom{n}{k}\}$ for a large class of black-box algorithms and tighten this to $\Omega(\min\{4^k/p_c, \binom{n}{k}\})$ for algorithms that always produce the same number $\lambda$ of offspring, regardless of whether crossover is used or not. Together, these lower bounds show that the \mulgavar has an asymptotically optimal expected optimisation time on $\jump_k$ and that it is provably faster than the original \muga by a factor of at least $\Omega(1/p_c)$, for all $k = O(\sqrt{n})$.

\section{Preliminaries}
\label{sec:preliminaries}

\subsection{Notation}
\newedit{\neweditx{By $\log$ we denote the logarithm to base $2$ and by $\ln$ the logarithm to base $e$, respectively}. For $m \in \N$ we write $[m]:= \{1,2,\ldots,m\}$.} For \newedit{$x \in \{0,1\}^n$} we denote by $\ones{x}$ the number of ones and by $\zeros{x}$ the number of zeros \newedit{in $x$}, respectively. For $x,y \in \{0,1\}^n$, the \emph{Hamming distance} \newedit{$H(x,y):=\sum_{i=1}^n|x_i-y_i|$} is the number of positions in which $x$ and $y$ differ. For $r \in [n]$ and $x,y \in \{0,1\}^n$ we define $H_r(x,y):=1$ if $x_r \neq y_r$ and $H_r(x,y)=0$ otherwise. 

 The fitness function $\jump_k$ is defined as
\[	
\jump_k(x) :=
\begin{cases}
	\ones{x}+k & \text{if } \ones{x} = n \text{ or }  \ones{x} \leq n-k, \\
	n - \ones{x} & \text{otherwise}\newedit{.}
\end{cases}
\]
By \plateau we mean the set of all search points with $n-k$ \newedit{ones}. 
We define $\jump_k'$ as a fitness function with $\jump_k'(x)=\jump(x)$ if ${x \neq 1^n}$ and $\jump_k'(x)=0$ otherwise. In $\jump_k'$\neweditx{,} jumps from $\plateau$ to $1^n$ \neweditx{result in a worse fitness and are thus rejected. Since the population remains on $\plateau$, this modification} allows us to study equilibria for population diversity on the plateau. 
Note that every $x \in \plateau$ optimises $\jump_k'$. Let $q:=\min(k,n-k)$. 

\subsection{Algorithms}
For $t \in \mathbb{N}_0$ let a population $P_t$ of $\mu$ individuals be \neweditx{a multiset} defined as $P_t:=\{x_1, \ldots , x_\mu\}$ where $x_i \in \{0,1\}^n$ for $i \in \{1,\ldots , n\}$.
We define the $(\mu+1)$ EA without crossover and the variant \mulgavar using crossover with crossover probability $p_c \in [0,1]$ \newedit{on such a population $P_t$}. 

The former starts with some initial population, selects a parent $x$ uniformly at random and creates an offspring $y$ by using a mutation operator on $x$. Then $y$ replaces a worst search point $z$ in the current population if its fitness is no worse than $z$. 

\begin{algorithm2e}[ht]
\DontPrintSemicolon
  $t \gets 0$\;
  \newedit{Let $P_0$ contain $\mu$ search points chosen uniformly at random}\;
  \While{termination criterion not met}{
    Select $x \in P_t$ uniformly at random\;
    $y \gets \mutation(x)$\;
    Select $z \in P_t$ uniformly at random from all search points with minimum fitness in $P_t$\;
    \If{$f(y) \ge f(z)$}{$P_{t+1} \gets (P_t \cup \{y\}) \setminus \{z\}$}
	$t \gets t+1$
}
\caption{The $(\mu+1)$ EA \newedit{maximising $f \colon \{0, 1\}^n \to \mathbb{R}$}}
\label{alg:steady-state-EA}
\end{algorithm2e}
With probability $p_c$ the \mulgavar picks two parents $x_1,x_2$ uniformly at random with replacement, generates $\lambdac$ offspring via crossover \newedit{on} $x_1$ and $x_2$, followed by a mutation operator. Then it chooses $y'$ as the best out of these $\lambdac$ offspring.
With probability $(1-p_c)$ it applies $\mutation$ to one parent chosen uniformly at random. The mutant $y'$ replaces a worst search point $z$ from the previous population if its fitness is no worse than $z$.
In the special case $\lambdac = 1$ we obtain the original \muga studied in previous work.


\begin{algorithm2e}[ht]
\DontPrintSemicolon
  $t \gets 0$\;
  \newedit{Let $P_0$ contain $\mu$ search points chosen uniformly at random}\;
  \While{termination criterion not met}{
    Choose \neweditx{$b\in [0,1)$} uniformly at random\;
    \uIf{\neweditx{$b < p_c$}}{
    Choose $x_1,x_2$ uniformly at random from $P_t$\label{line:crossover-in-steady-state-GA}\;
    \For{$i=1$ \KwTo $\lambdac$}{
        $y_i \gets \mathrm{crossover}(x_1,x_2)$\;
        $y_i' \gets \mutation(y_i)$\;
        }
    Choose $y'$ uniformly at random from all $y_i'$ with maximum fitness in $\{y_1', \dots, y_\lambdac'\}$\;
    }
    \Else{
    Select $y \in P_t$ uniformly at random\;
        $y' \gets \mutation(y)$\;
    }
    Select $z \in P_t$ uniformly at random from all search points with minimum fitness in $P_t$\;
    \If{$f(y') \ge f(z)$}{$P_{t+1} \gets (P_t \cup \{y'\}) \setminus \{z\}$}
	$t \gets t+1$
}
\caption{The \protect\mulgavar with crossover probability $p_c \in [0,1]$ maximising $f \colon \{0, 1\}^n \to \mathbb{R}$}
\label{alg:steady-state-GA}
\end{algorithm2e}

Mutation operators will be chosen from the class of \emph{unary unbiased variation operators}~\cite{Lehre2012}, or \emph{unbiased mutation operators} for short, that treat all bit values and all bit positions symmetrically (see~\cite{Lehre2012} for a formal definition). A well-known implication is the following.
\begin{lemma}
\label{lem:bit-flip-probability-unbiased}
Consider an unbiased mutation operator that flips $\chi$ bits in expectation. Then\neweditx{,} for any arbitrary but fixed bit~$i$, the probability that it is being flipped is $\chi/n$.
\end{lemma}
The operator flipping every bit \emph{independently} with probability $\chi/n$ is called \emph{standard bit mutation}.
Note that in other mutation operators (\newedit{e.g.}\ flipping exactly $\chi \in \mathbb{N}_0$ bits), each fixed bit is also flipped with probability $\chi/n$, but these decisions may not be independent for all bits. \newedit{A well-known crossover operator which we use throughout this paper is uniform crossover, in which every bit is taken from a parent chosen uniformly at random.} 



\subsection{Diversity Measure}
\label{sec:diversity}
In this paper we use the sum of Hamming distances as a diversity measure which is defined in the following\newedit{.}
\begin{definition}
\label{def:Hamming-distance}
For a population $P_t = \{x_1, \dots, x_{\mu}\}$ and a search point $y \in \{0, 1\}^n$ we define 
\[
    S_{P_t}(y) \coloneqq \sum\nolimits_{i=1}^{\mu} H(x_i,y)
\quad \text{and} \quad
    S(P_t) \coloneqq \sum\nolimits_{i=1}^\mu \sum\nolimits_{j=1}^\mu H(x_i, x_j),
\]
and for $r\in [n]$ we write $S_r(P_t):=\sum\nolimits_{i=1}^{\mu} \sum\nolimits_{j=1}^{\mu}H_r(x_i,x_j)$, so that $S(P_t) = \sum\nolimits_{r=1}^n S_r(P_t)$.
\end{definition}

The population diversity can be computed by considering each bit position individually and counting how many population members have a zero or a one, respectively.
\begin{lemma}[Theorem~1a in \cite{wineberg2003underlying}]
Let $P_t$ be a population of size~$\mu$ and $m_i$ denote the number of members of~$P_t$ having a zero at position~$i$, then $S(P_t) = \sum\nolimits_{i=1}^n 2m_i(\mu-m_i)$.    
\end{lemma}
Note the symmetry of zeros and ones in the expression $m_i(\mu-m_i)$. The factor 2 accounts for the fact that every pair $(x_i, x_j)$ with $i \neq j$ is counted twice in $\sum\nolimits_{i=1}^\mu \sum\nolimits_{j=1}^\mu H(x_i, x_j)$. According to~\cite{Lengler2024}, $S(P_t)$ is twice the sum of all Hamming distances of each pair $x_i, x_j$ with $i \neq j$, the average value of $S_{P_t}(y)$ \newedit{over} $y\in P_t$ is $S(P_t)/\mu$ and the average Hamming distance of two random points in $P_t$ is $S(P_t)/\mu^2$. The latter implies that if $S(P_t)>a\mu^2$ for $a \in \mathbb{R}$, \neweditx{then} there is at least \newedit{one} pair $(x_1,x_2)$ of individuals with $H(x_1,x_2) > a$. Note that the average Hamming distance of two \emph{different} random 
search points is $S(P_t)/(\mu^2-\mu)$.
For populations $P_t \subset \plateau$, the maximum Hamming distance between two different search points is $2k$ and thus \newedit{by summing $H(x_i,x_j)$ \neweditx{over} all search points we obtain} ${S(P_t) \le 2k\mu(\mu-1)}$ \newedit{since there are $\mu^2$ summands in total and the $\mu$ summands of the form $H(x_i,x_i)$ for $i \in [\mu]$ are zero}. More precisely, the maximum diversity is obtained if the zeros in the population are spread across bit positions as evenly as possible, \newedit{i.e.} if every bit position has a zero in $\lfloor k\mu/n\rfloor$ or $\lceil k\mu/n \rceil$ individuals.


\section{Equilibrium State for $S(P_t)$ on \textsc{Plateau} for the \muea}
\label{sec:equilibrium-states}

Now we will compute the expected change of $S(P_t)$, i.e., $\E(S(P_{t+1}) \mid S(P_t))$ if all individuals from $S(P_t)$ are in \plateau. To this end\neweditx{,} we investigate the population dynamics of evolutionary algorithms on the modification $\jump_k'$ of the $\jump_k$ \newedit{benchmark} where the global optimum is removed. 

We break the process down into several steps.
Even though our main goal is to analyse the \mulgavar with standard bit mutation as described in Algorithm~\ref{alg:steady-state-GA}, we make an effort to keep structural results as general as possible as they might be of independent interest. In particular, we start by considering arbitrary mutation operators, then focus on unbiased mutation and finally arrive at standard bit mutation. For the same reason, we formulate many statements for $q = \min\{k, n-k\}$, which considers jumps of \newedit{size} $k > n/2$, even though our runtime bounds are limited to $k = o(n)$, that is, $q=k$.


We start with a lemma which describes, for given $x,y \in $ \plateau, the effect of flipping $\ell$ random \neweditx{zero}-bits and $\ell$ random \neweditx{one}-bits of $y$ on the distance from $x$. Note that this yields an offspring in \plateau.

\begin{lemma}
\label{lem:how-H-is-derived}
Suppose that $x,y \in \plateau$ (i.e. $\ones{x} = \ones{y} = n-k$) and let $z$ be generated from $y$ by flipping $\ell$ ones and $\ell$ zeros chosen uniformly at random for $\ell \in \{0, \dots , q\}$. Then
\begin{align*}
\E(H(x,z)) & = 2\ell+ \big(1-\tfrac{\ell n}{k(n-k)}\big)H(x,y).
\end{align*}
\end{lemma}

\begin{proof}
We write $R:= \{i\in[n] \mid y_i = 0\}$ and $R^c := [n]\setminus R$. Note that $|R|=k$ and $R^c=n-k$ since $y$ is in \plateau. Moreover, for any $x_1,x_2\in \{0,1\}^n$ we write $H_R(x_1,x_2)$ for the number of bits \emph{in $R$} in which $x_1$ and $x_2$ differ, and similarly for $R^c$. In particular, $H(x_1,x_2) = H_R(x_1,x_2) + H_{R^c}(x_1,x_2)$. Moreover, since $x$ and $y$ are both in \plateau, we have $H_R(x,y) = H_{R^c}(x,y) = H(x,y)/2$.

Consider $z_0,z_1$ where 
\begin{itemize}
\item $z_0$ \neweditx{is} generated from $y$ by flipping $\ell$ zeros chosen uniformly at random, \neweditx{and}
\item $z_1$ \neweditx{is} generated from $y$ by flipping $\ell$ ones chosen uniformly at random.
\end{itemize}
Note that $\E(H(x,z_0)) = \E(H_R(x,z))+H_{R^c}(x,y)$ since $z_0$ is obtained from $y$ by only flipping zeros and the set $R^c$ is untouched by the mutation. Similarly we see $\E(H(x,z_1)) = \E(H_{R^c}(x,z)) + H_R(x,y)$. Therefore, 
\begin{align*}
\E(H(x,z)) = \E(H_R(x,z)) + \E(H_{R^c}(x,z)) = \E(H(x,z_0)) + \E(H(x,z_1)) - H(x,y).
\end{align*}

Hence, it is enough to show 
\begin{align*}
\E(H(x,z_0)) & = \ell+ \left(1-\frac{\ell}{k}\right)H(x,y),\\
\E(H(x,z_1)) & = \ell+ \left(1- \frac{\ell}{n-k}\right)H(x,y).
\end{align*}

We start with $z_0$. For the $k$ positions in $R$, each of them is flipped with probability $\ell/k$. There are $H_R(x,y)$ bits in $R$ in which $x$ and $y$ differ, and for any such bit the Hamming distance decreases by one if the bit is flipped. Likewise, there are $|R|-H_R(x,y)$ bits in $R$ in which $x$ and $y$ coincide, and for any such bit the Hamming distance increases by one if the bit is flipped. Combining these two cases yields
\begin{align*}
\E(H_R(x,z_0)) &= H_R(x,y) + \frac{\ell}{k}\cdot\left(-H_R(x,y) + (|R|-H_R(x,y))\right)\\
&= H_R(x,y) + \frac{\ell}{k}\cdot\left(|R|-2H_R(x,y)\right).
\end{align*}
Plugging in $H_R(x,y) = H(x,y)/2$ and $|R| = k$ yields after simplifying 
\begin{align}\label{eq:flipping-only-zero-bits}
\E(H_R(x,z_0)) & = \ell + \left(\frac{1}{2}-\frac{\ell}{k}\right)H(x,y),
\end{align}
and hence
\begin{align*}
\E(H(x,z_0)) & = H_{R^c}(x,z_0) + \E(H_R(x,z_0))\\
& = \frac{H(x,y)}{2} + \ell + \left(\frac{1}{2}-\frac{\ell}{k}\right)H(x,y)\\
& = \ell + \left(1-\frac{\ell}{k}\right)H(x,y).
\end{align*}
For $z_1$, the roles of $R$ and $R^c$ are reversed, with the notable difference that the probability of flipping each bit in $R^c$ is now $\ell/|R^c|=\ell/(n-k)$. The formula 
\begin{align}\label{eq:flipping-only-one-bits}
\E(H_{R^c}(x,z_1)) & = \ell + \left(\frac{1}{2}-\frac{\ell}{n-k}\right)H(x,y)
\end{align}
and the final result for $\E(H(x,z_1))$ follows analogously.
\end{proof}

For further steps we need the following lemma from \cite{Lengler2024} to describe the expected change of $S(P_t)$ for a fixed offspring $y'$. This result works on all fitness functions.


\begin{lemma}[\cite{Lengler2024}, Lemma~6]
\label{lem:how-E-S-P-t-plus-1-is-derived}
Consider a population $P_t = \{x_1, \dots, x_\mu\}$ and a search point $y'\in\{0,1\}^n$. Let $P_{t+1} \coloneq (P_t \cup \{y'\}) \setminus \{x_d\}$ for a uniform random $d \in [\mu]$. Then
\[
\E(S(P_{t+1}) \mid S(P_t), y') = \left(1 - \tfrac{2}{\mu}\right) S(P_t) + \tfrac{2(\mu-1)}{\mu} S_{P_t}(y').
\]
\end{lemma}


Now we combine Lemma~\ref{lem:how-H-is-derived} with Lemma~\ref{lem:how-E-S-P-t-plus-1-is-derived} to describe the expected change of $S(P_t)$ if $y'$ arises from another $y \in P_t$ by a mutation operator which flips $\ell$ ones and zeros chosen uniformly at random and every $x \in P_t$ is in \plateau.


\begin{lemma}
\label{lem:how-E-S-P-t-plus-1-is-derived2}
Let $P_t \subset \plateau$ be a population of size~$\mu$. Let $y \in P_t$ \newedit{be fixed} and suppose that $y'$ results from $y$ by a mutation operator which flips $\ell$ ones and $\ell$ zeros chosen uniformly at random. Let $P_{t+1}=(P_t \cup \{y'\}) \setminus \{x_d\}$ for a uniform random $d \in [\mu]$. Then
\begin{align*}
& \E(S(P_{t+1}) \mid S(P_t)) = 
\left(1 - \tfrac{2}{\mu}\right) S(P_t) + 4 \ell (\mu-1) + \tfrac{2(\mu-1)}{\mu} \Big(1-\tfrac{\ell n}{k(n-k)}\Big) S_{P_t}(y).
\end{align*}
\end{lemma}
\begin{proof}
For $z \in P_t$ we have $S(z):=S_{P_t}(z)=\sum\nolimits_{i=1}^{\mu}H(x_i,z)$. By Lemma~\ref{lem:how-E-S-P-t-plus-1-is-derived} and the law of total probability 
we obtain
\begin{equation}
\label{eq:how-E-S-P-t-plus-1-is-derived2-eq1}
\E(S(P_{t+1}) \mid S(P_t)) = \left(1 - \frac{2}{\mu}\right) S(P_t) + \frac{2(\mu-1)}{\mu} \E(S(y')).
\end{equation}
By Lemma~\ref{lem:how-H-is-derived} we obtain for all $i \in [\mu]$
\[
\E(H(x_i,y')) = 2\ell + \left(1-\frac{\ell n}{k(n-k)}\right)H(x_i,y).
\]
Using linearity of expectation,
\[
\E(S(y')) = \sum\nolimits_{i=1}^\mu \E(H(x_i,y')) = 2 \ell \mu + \left(1-\frac{\ell n}{k(n-k)}\right) S(y)
\]
and plugging this into~\eqref{eq:how-E-S-P-t-plus-1-is-derived2-eq1} yields the claim.
\end{proof}

Remarkably, the right-hand side
only depends on $S_{P_t}(y)$, not on $y$ itself. This means that all $y$ with the same value of $S_{P_t}(y)$ yield the same dynamics. The following corollary describes how the diversity evolves in the case of a $(\mu+1)$ EA with a mutation operator which flips $\ell$ ones and zeros chosen uniformly at random. This result transfers straightforward\newedit{ly} to unbiased mutation operators if one knows the probability $p_\ell$ to flip $\ell$ ones and $\ell$ zeros. 

\begin{corollary}
\label{cor:diversity-EA}
Let $\ell \in \{0, \ldots , q\}$ (where $q=\min(k,n-k)$).
Consider \newedit{the} $(\mu+1)$ EA (see Algorithm~\ref{alg:steady-state-EA}) on $\jump_k$ with a mutation operator which flips $\ell$ ones and $\ell$ zeros chosen uniformly at random. Suppose its current population is $P_t \subset \plateau$.
Then
\begin{align*}
& \E(S(P_{t+1}) \mid S(P_t)) 
= \left(1 - \frac{2}{\mu^2} - \frac{2(\mu-1)\ell n}{\mu^2 k(n-k)}\right) S(P_t) + 4 \ell (\mu-1).
\end{align*}
\end{corollary}

\begin{proof}
By \newedit{Lemma~\ref{lem:how-E-S-P-t-plus-1-is-derived2} and} the law of total probability we have
\begin{align*}
\E(S(P_{t+1}) \mid S(P_t)) 
&= \left(1 - \tfrac{2}{\mu}\right) S(P_t) + 4 \ell (\mu-1) + \tfrac{2(\mu-1)}{\mu} \Big(1-\tfrac{\ell n}{k(n-k)}\Big) \E(S(y))
\end{align*}
where $y$ denotes the parent chosen uniformly at random. With
\[
\E(S(y)) = \frac{1}{\mu} \sum\nolimits_{i=1}^\mu \sum\nolimits_{j=1}^\mu H(x_i,x_j) = S(P_t)/\mu
\]
we obtain the result.
\end{proof}

Now we extend our analysis to arbitrary unbiased mutation operators.
\begin{corollary}
\label{cor:diversity-EA2}
Consider a $(\mu+1)$ EA (see Algorithm \ref{alg:steady-state-EA}) on $\jump_k'$ with a current population $P_t \subset \plateau$ and an unbiased mutation operator \newedit{$\mut$}. Then 
\begin{align*} 
& \E(S(P_{t+1}) \mid S(P_t)) \\
&= S(P_t)
 - \sum\nolimits_{\ell=0}^q p_\ell \left(\frac{2}{\mu^2} + \frac{2(\mu-1)\ell n}{\mu^2 k(n-k)}\right)S(P_t) + 4(\mu-1)\sum\nolimits_{\ell=0}^q \ell p_\ell
\end{align*}
where $p_\ell$ denotes the probability that \newedit{$\mut$} flips $\ell$ ones and $\ell$ zeros where $\ell \in \{0, \ldots ,q\}$.
\end{corollary}

\begin{proof}
For $\ell \in \{0, \ldots , q\}$ denote by $A_\ell$ the event that $\ell$ zeros and $\ell$ ones are flipped. Since \newedit{$\mut$} is unbiased, it flips $\ell$ ones and $\ell$ zeros chosen uniformly at random if $A_\ell$ occurs. Since every individual $x \in P_t$ is in \plateau (i.e. has $k$ zeros and $n-k$ ones), we see that the outcome $y$ is in \plateau if and only if one of the events $A_0,A_1, \ldots ,A_q$ occurs. Let $B$ be the event that $y$ has a worse fitness than every $x \in P_t$, i.e. $y$ is not in \plateau. 
Then $\pr(B)+\sum\nolimits_{\ell=0}^q p_\ell = 1$ and by the law of total probability
\begin{align*}
 \neweditx{\E(S(P_{t+1}) \mid S(P_t)) =}\;&\neweditx{\E(S(P_{t+1}) \mid S(P_t),B)\cdot \pr(B \mid S(P_t))} \\
&\neweditx{+ \sum\nolimits_{\ell=0}^q \E(S(P_{t+1}) \mid S(P_t), A_\ell)\cdot p_\ell.}
\end{align*}
If $B$ occurs, $S(P_t)$ does not change, i.e. $\E(S(P_{t+1}) \mid S(P_t),B) = S(P_t)$. Since 
\begin{align*}
&\E(S(P_{t+1}) \mid S(P_t),A_\ell) 
= \left(1 - \frac{2}{\mu^2} - \frac{2(\mu-1)\ell n}{\mu^2 k(n-k)}\right) S(P_t) + 4 \ell (\mu-1)
\end{align*}
by Corollary~\ref{cor:diversity-EA}, we obtain the result. 
\end{proof}

As for flat fitness landscapes in \cite{Lengler2024} (i.e. the fitness function is constant over the whole search space), Corollary~\ref{cor:diversity-EA2} shows that there is an equilibrium state in \plateau. We define
\begin{align}
\label{eq:betagamma}
\beta:=4(\mu-1)\sum\nolimits_{\ell=0}^q \ell p_\ell \text{ \ and \ } \gamma:= \sum\nolimits_{\ell=0}^q p_\ell \left(\frac{2}{\mu^2} + \frac{2(\mu-1)\ell n}{\mu^2 k(n-k)}\right).
\end{align}
Then\neweditx{, for all populations $P_t \subset \plateau$,}
\begin{align}
\label{eq:expectation}
\E(S(P_{t+1}) \mid S(P_t)) = (1-\gamma)S(P_t)+\beta.
\end{align}
As in \cite{Lengler2024} this condition is a negative multiplicative drift with an additive disturbance (this terminology stems  from~\cite{DoerrNegativeMultiplicativeDrift}). An equilibrium state for $S(P_t)$ with zero drift is attained for
\begin{align}
S_0:=\frac{\beta}{\gamma} &= \frac{4(\mu-1)\sum\nolimits_{\ell=0}^{q} \ell p_\ell}{\frac{2}{\mu^2}\left(\sum\nolimits_{\ell=0}^q p_\ell\right) + \frac{2(\mu-1)n}{\mu^2 k(n-k)} \sum\nolimits_{\ell=0}^q \ell p_\ell}
= \frac{2(\mu-1) \mu^2}{\frac{\sum\nolimits_{\ell=0}^q p_\ell}{\sum\nolimits_{\ell=0}^q \ell p_\ell} + \frac{(\mu-1)n}{k(n-k)}}\label{eq:equilibrium-for-EA-nonsimplified}
\end{align}
if \newedit{$(p_1,\ldots,p_q) \neq 0$} and $S_0=S(P_0)$ otherwise.
The major difference to \cite{Lengler2024} is that this equilibrium depends not only on the number of bits flipped in expectation, but also on the probabilities $p_\ell$ to flip exactly $\ell$ ones and $\ell$ zeros. For 
bitwise mutation with mutation \neweditx{probability $\chi/n$ where $0 \leq \chi \leq n$}.
\begin{align}
p_\ell 
&= \binom{k}{\ell} \cdot \binom{n-k}{\ell} \cdot \left(\frac{\chi}{n}\right)^{2\ell} \cdot \left(1-\frac{\chi}{n}\right)^{n-2\ell}.
\end{align}

The following lemma gives the equilibrium state for standard bit mutation with mutation probability $\chi/n$ if $\chi \in \neweditx{O}(1)$ \neweditx{and $\chi > 0$}. Recall that $q=\min(k,n-k)$. 

\begin{lemma}
\label{lem:Estimate-p-binomial}
Consider standard bit mutation with a mutation probability $\chi/n$ with $\chi \in O(1)$.  Then $\sum\nolimits_{\ell=0}^q p_\ell = e^{-\chi}(1+\nu_1)$ and $\sum\nolimits_{\ell=0}^q \ell p_\ell = k(n-k) (\chi/n)^2 \cdot e^{-\chi}(1+\nu_2)$ for $|\nu_1|,|\nu_2| \in O(q/n)$. \neweditx{Therefore, if $\chi>0$,}
\[
\frac{\sum\nolimits_{\ell=0}^q p_\ell}{\sum\nolimits_{\ell=0}^q \ell p_\ell} = \frac{n^2}{k \chi^2(n-k)} (1+\xi) 
\]
for $|\xi| \in O(q/n)$. 
\end{lemma}
\begin{proof}
We have
\begin{align*}
&\sum\nolimits_{\ell=0}^q \ell p_\ell = \frac{\chi^2}{n^2}\Big(1-\frac{\chi}{n}\Big)^{n-2}\sum\nolimits_{\ell=1}^{q} \Big(\ell \binom{k}{\ell} \binom{n-k}{\ell} \cdot \Big(\frac{\chi}{n-\chi}\Big)^{2\ell-2}\Big)\\
&= \frac{k(n-k)\chi^2}{n^2}\Big(1-\frac{\chi}{n}\Big)^{n-2}\sum\nolimits_{\ell=1}^q \Big(\frac{\chi^{2\ell-2} \prod_{j=1}^{\ell-1}(k-j)(n-k-j)}{(n-\chi)^{2\ell-2}(\ell-1)! \ell!}\Big)\\
&= \frac{k(n-k)\chi^2}{n^2}\Big(1-\frac{\chi}{n}\Big)^{n-2}\Big(1+\frac{q}{n} \sum\nolimits_{\ell=2}^q\frac{n\chi^{2\ell-2}\prod_{j=1}^{\ell-1}(k-j)(n-k-j)}{q(n-\chi)^{2\ell-2}(\ell-1)! \ell!}\Big).
 \end{align*}
We may bound $(1-\chi/n)^{n-2} = e^{-\chi}(1+O(\chi/n))$ by~\cite[Corollary~1.4.6]{doerr-theory-chapter} and swallow the latter factor in $\nu_2$. To bound the remaining terms, we use 
\[
\sum\nolimits_{\ell=2}^q\frac{\chi^{2\ell-2}}{(\ell-1)! \ell!} = \sum\nolimits_{\ell=1}^{q-1}\frac{\chi^{2\ell}}{\ell! (\ell+1)!} \leq \sum\nolimits_{\ell=0}^{q-1}\frac{\chi^{2\ell}}{\ell!} \leq e^{\chi^2} \in O(1)
\]
and for every $\ell \in \{2, \dots ,q\}$
\[
\frac{n\prod_{j=1}^{\ell-1}(k-j)(n-k-j)}{q(n-\chi)^{2\ell-2}}
\le \frac{n(k-1)^{\ell-1}(n-k-1)^{\ell-1}}{q(n-\chi)^{2\ell-2}}. 
\]
Since $(k-1)(n-k-1) = (q-1)(n-q-1) \le qn$, we may bound the latter as
\begin{align*}
&\leq \frac{n^2q(k-1)^{\ell-2}(n-k-1)^{\ell-2}}{q(n-\chi)^{2\ell-2}} \leq \Big(\frac{n}{n-\chi}\Big)^{2 \ell - 2} \\
&\neweditx{\; \leq \;} \Big(1+\frac{\chi}{n-\chi}\Big)^{2 \ell} \leq e^{\frac{\chi}{n-\chi} \cdot 2\ell} \leq e^{\frac{\chi}{n-\chi} \cdot 2k} \in O(1). 
\end{align*}
In an analogous way we obtain the result for $\sum\nolimits_{\ell=0}^q p_\ell$. 
We have 
\begin{align*}
\sum\nolimits_{\ell=0}^q p_\ell &= \left(1-\frac{\chi}{n}\right)^n\sum\nolimits_{\ell=0}^q\binom{k}{\ell} \cdot \binom{n-k}{\ell} \cdot \left(\frac{\chi}{n-\chi}\right)^{2\ell}\\
&=\Big(1-\frac{\chi}{n}\Big)^n\sum\nolimits_{\ell=0}^q \Big(\frac{\chi^{2\ell} \prod_{j=0}^{\ell-1}(k-j)(n-k-j)}{(n-\chi)^{2\ell}\ell! \ell!}\Big)\\
&=\Big(1-\frac{\chi}{n}\Big)^n\left(1+\frac{q}{n}\sum\nolimits_{\ell=1}^q \Big(\frac{n\chi^{2\ell} \prod_{j=0}^{\ell-1}(k-j)(n-k-j)}{q(n-\chi)^{2\ell}\ell! \ell!}\Big)\right).
\end{align*}
As before, we may estimate $(1-\chi/n)^{n} = e^{-\chi}(1-O(\chi/n))$. Moreover, 
\[
\sum\nolimits_{\ell=1}^q\frac{\chi^{2\ell}}{\ell! \ell!} \leq \sum\nolimits_{\ell=1}^q\frac{\chi^{2\ell}}{\ell!} \leq e^{\chi^2} \in O(1)
\]
and for every $\ell \in \{1 \ldots ,q\}$
\begin{align*}
\frac{n\prod_{j=0}^{\ell-1}(k-j)(n-k-j)}{q(n-\chi)^{2\ell}} &= \frac{nq(n-q) \prod_{j=1}^{\ell-1}(k-j)(n-k-j)}{q(n-\chi)^{2\ell}}\\
&\leq \frac{n^{2\ell}}{(n-\chi)^{2\ell}} \leq e^{\frac{\chi}{n-\chi} \cdot 2k} \in O(1).\qedhere
\end{align*}

\end{proof}

Now we can state the equilibrium for standard bit mutation.
Plugging the simplifications from Lemma~\ref{lem:Estimate-p-binomial} into~\eqref{eq:equilibrium-for-EA-nonsimplified},
\begin{align*}
    S_0 =\;& \frac{2(\mu-1) \mu^2}{\frac{n^2}{k \chi^2(n-k)} (1+\xi) + \frac{(\mu-1)n}{k(n-k)}}\newedit{.}
\end{align*}
Thus, we get the following.
\begin{corollary}
For the \muea on \jump$_k'$ with standard bit mutation and mutation \neweditx{probability} $\chi/n$ \newedit{with $\chi \in O(1)$}, starting with an arbitrary population $P_t \subset \plateau$, the equilibrium state for population diversity is
\[
S_0 \coloneqq \frac{\beta}{\gamma} = \frac{2\mu^2\frac{k(n-k)}{n}}{\frac{n}{(\mu-1)\chi^2}(1+\xi)+1}.
\]
\end{corollary}
To obtain the average Hamming distance in the population, we need to divide $S_0$ by $\mu^2$. If $k = o(n)$ the numerator in $S_0/\mu^2$ is $2k(1-o(1))$ and if $\mu \chi^2 = \omega(n)$, the denominator is $1+o(1)$. Then the average Hamming distance is $2k(1 - o(1))$, that is, converging to its maximum value~$2k$.

\section{Analysing Population Diversity on Plateau for the \mulgavar}
\label{sec:equilibrium-states2}

We now focus on the \mulgavar, which performs crossover with probability with probability $p_c$ and generates $\lambdac$ crossover offspring from the same parents, as outlined in Algorithm~\ref{alg:steady-state-GA}. \neweditx{The goal of this section is to estimate $\E(S(P_{t+1}) \mid S(P_t))$ and to derive an equilibrium state for the \mulgavar on $\jump_k'$, assuming all individuals are on \plateau. More precisely, we will show a lower bound on the equilibrium state, which will later be used to prove lower bounds on the population diversity $S(P_t)$.}


\neweditx{
\textbf{Roadmap:} 
The key result of Section~\ref{sec:fixed-parents} is Theorem~\ref{thm:crossover-specific-parents-fixed}, which gives a lower bound on the population diversity $S(P_{t+1})$ for a fixed choice of parents $x_1, x_2$ and conditional on the event that the resulting offspring remains on the plateau. To achieve this, we analyse the effect of a \emph{balanced uniform crossover}~\cite{Friedrich2023} instead of the standard uniform crossover. Balanced uniform crossover ensures that the offspring has the same number of one-bits as both parents, guaranteeing that it remains on the plateau and is thus accepted into the population.
To relate the two crossover operators, we observe that uniform crossover can be decomposed into a two-stage process: first, applying balanced uniform crossover to the parents, and second, mutating the offspring by flipping bits to match the offspring distribution of standard uniform crossover. 
A key property of balanced uniform crossover, shown in~\cite{Lengler2024}, is that it is \emph{diversity-neutral}, meaning that the expected change in population diversity remains the same whether or not crossover is applied. This property allows us to use the analytical framework established in our previous work~\cite{Lengler2024}, which focuses on determining an equilibrium state for $S(P_t)$.
}

\neweditx{
In Section~\ref{sec:equilibrium-state-crossover} we then derive a lower bound on $E(S(P_{t+1}))$ for one generation of the \mulgavar (Lemma~\ref{lem:crossover-specific}). This argument relies on the fact that the \mulgavar uses competing crossover offspring for a parameter $\lambda_c$ large enough to guarantee that the best offspring is on the plateau, with high probability. We use this result to estimate a lower bound on the equilibrium state for population diversity (Corollary~\ref{cor:alphadelta}). To justify the design of the \mulgavar, we further provide a counterexample proving that the results from this section do not hold for the original \muga (Lemmas~\ref{lem:negative-standard-GA} and \ref{lem:negative-standard-GA-2}).
}

\neweditx{
Finally, in Section~\ref{sec:diversity-close-maximum} we prove that the \mulgavar typically maintains a population diversity close to its equilibrium in at least a constant fraction of the time (Lemma~\ref{lem:SPt-is-large}).
}

\subsection{The Dynamics of $S(P_t)$ for Fixed Parents}
\label{sec:fixed-parents}

\neweditx{As outlined in our roadmap, we begin our analysis of the dynamics of $S(P_t)$ by considering balanced uniform crossover, as it fits the} notion of \emph{diversity-neutral} crossover defined in~\cite{Lengler2024}. In the context of~\cite{Lengler2024} (that is, on flat fitness landscapes) \newedit{a diversity-neutral crossover} guaranteed that the equilibrium of a mutation-only EA was not affected when introducing \newedit{such a} crossover operator.
\begin{definition}[\cite{Lengler2024}]
\label{def:diversity-neutral}
We call a crossover operator $\cross$ \emph{diversity-neutral} if it has the following property. For all $x_1,x_2,z \in \{0, 1\}^n$,
\begin{equation}
    \label{eq:crossover-property}
\E(H(\cross(x_1, x_2), z) + H(\cross(x_2, x_1), z)) = H(x_1, z) + H(x_2, z).
\end{equation}
\end{definition}
Many well-known operators were shown to be diversity-neutral in~\cite{Lengler2024}, including uniform crossover, $k$-point crossover and balanced uniform crossover~\cite{Friedrich2023}. The latter copies bit values on which both parents agree. If the parents differ in $k$ positions, it then chooses values for these bits uniformly at random from all sub-strings that have exactly $\lfloor k/2 \rfloor$ ones at these positions. \newedit{Hence, it guarantees that amongst those $k$ positions there will be exactly $\lfloor{k/2}\rfloor$ ones.}
 Balanced uniform crossover will appear as a vehicle in \neweditx{the proof of the main result of this subsection, Theorem~\ref{thm:crossover-specific-parents-fixed}.}

We start our analysis with the following statement from \cite{Lengler2024}, which holds for arbitrary diversity-neutral crossover operators. 
\neweditx{Note that Theorem~18 in~\cite{Lengler2024} considers the application of a diversity-neutral crossover, followed by an unbiased mutation operator that flips $\chi$ bits in expectation. Here we omit mutation, or equivalently, consider a trivial ``mutation'' operator that always flips $\chi=0$ bits. Then Theorem~18 in~\cite{Lengler2024}}
considers a population $Q_t$ that results from creating an offspring \neweditx{via a diversity-neutral crossover} and always accepting, hence simulating one step on a flat landscape. Note that $Q_t$ is in general not the same as $P_{t+1}$; they are equal only in steps in which the offspring is accepted. 
\begin{lemma} (\cite[Theorem~18]{Lengler2024} for $\chi=0$.)
\label{lem:crossover-general}
For every population $P_t$ of size~$\mu$ the following holds. Suppose that $y \in \{0,1\}^n$ is the outcome of a diversity-neutral crossover $c$ on two parents chosen uniformly at random from $P_t$. Let $\neweditx{Q_t^b} :=(P_t \cup \{y\}) \setminus \{x_d\}$ for a uniform random $d \in [\mu]$. Then
\[
\E(S(\neweditx{Q_t^b}) \mid S(P_t)) = \big(1-\tfrac{2}{\mu^2} \big)S(P_t).
\]
\end{lemma}

We also need the following easy computational lemma. Recall that $S_i(P) = \sum\nolimits_{x\in P}\sum\nolimits_{x'\in P} H_i(x,x')$, where $H_i(x,x')$ is the indicator for the event $x_i\neq x_i'$.

\begin{lemma}
\label{lem:diversity-general2}
    Let $i\in [n]$, let $P_t \subset \plateau$ be a population of size~$\mu$ and let $y,z \in \{0,1\}^n$ with $y_i \neq z_i$. Let $P_1:=(P_t \cup \{y\}) \setminus \{x_d\}$ and $P_2:=(P_t \cup \{z\}) \setminus \{x_d\}$ where $d \in [\mu]$ is chosen uniformly at random. Let $m$ be the number of bits with value~$z_i$ at position $i$ in $P_t$. Then 
\begin{align*}
\E(S_i(P_2)) = \E(S_i(P_1))+ 2\mu-4m(1 - 1/\mu)-2 \ge \E(S_i(P_1)) - 2(\mu-1).
\end{align*} 
\end{lemma}
\begin{proof}
    The quantity of $z_i$ bits at position $i$ in $P_t \cup \{z\}$ is by one larger than in $P_t \cup \{y\}$. Since $S_i(P_t) = 2m(\mu-m)$, we obtain $S_i(P_1) =2(m-1)(\mu-m+1)$ and $S_i(P_2) = 2m(\mu-m)$ if $(x_d)_i =z_i$ (which happens with probability $m/\mu$), and $S_i(P_1) =2m(\mu-m)$ and $S_i(P_2) = 2(m+1)(\mu-m-1)$ if $(x_d)_i =y_i \neq z_i$ (which happens with probability $(\mu-m)/\mu$). This yields
    \[
    \E(S_i(P_1)) = \frac{\mu-m}{\mu} \cdot 2m(\mu-m) + \frac{m}{\mu} \cdot 2(m-1)(\mu-m+1)
    \]
    and 
    \[
    \E(S_i(P_2)) = \frac{\mu-m}{\mu} \cdot 2(m+1)(\mu-m-1) + \frac{m}{\mu} \cdot 2m(\mu-m).
    \]
    Thus 
    \begin{align*}
    \E(S_i(P_2))-\E(S_i(P_1)) &= \frac{\mu-m}{\mu} \cdot 2(\mu-2m-1) + \frac{m}{\mu} \cdot 2(\mu-2m+1)\\
                      &= 2\mu-4m + \frac{4m-2\mu}{\mu}\\
                      &= 2\mu-4m\left(1 - \frac{1}{\mu}\right) -2 \geq -2(\mu-1)
    \end{align*}
    using $m \le \mu$ in the last step.
\end{proof}

In the next theorem we show that conditional on being in \plateau, the loss of diversity is not larger than \neweditx{$9k\mu\chi/(n \pr(V))$} compared to balanced uniform crossover\neweditx{, where  $\pr(V)$ is the probability of an offspring created via standard uniform crossover and mutation remaining on the plateau}. This is the key step of our analysis \neweditx{and the main result of this subsection}.

\begin{theorem}
\label{thm:crossover-specific-parents-fixed}
Suppose that $k \in o(n)$ \neweditx{ and $P_t \subset \textsc{plateau}$}. For every pair of parents $x_1, x_2 \in P_t$ the following holds.
Consider a random offspring $y \in \{0,1\}^n$ resulting from uniform crossover on $(x_1, x_2)$. Let $z$ be a random search point created from $y$ by standard bit mutation with mutation probability $\chi/n = \Theta(1/n)$. Let $d \in [\mu]$ \neweditx{be} uniform random and let $V$ be the event $z \in \plateau$. Let $P_{t+1} :=(P_t \cup \{z\}) \setminus \{x_d\}$. Let $y^b$ be the random search point resulting from balanced uniform crossover on $(x_1, x_2)$.
Then for $Q_t^b:=(P_t \cup \{y^b\}) \setminus \{x_d\}$ 
and sufficiently large $n$
\begin{align*}
\E(S(P_{t+1}) \mid V) \geq \E(S(Q_t^b)) - \frac{9k\mu\chi}{n \pr(V)}.
\end{align*}
\end{theorem}


\begin{proof}
Call a position $i$ \emph{disagreeing} if $x_1$ and $x_2$ differ at $i$ (i.e. $(x_1)_i \neq (x_2)_i$). Note that \neweditx{$H(x_1,x_2)$ is even and that} there are $H(x_1, x_2) \le 2k$ disagreeing positions. 
For $\ell \in \{-k, \ldots , k\}$ denote by $A_\ell$ the event that, in the substring $y'$ of $y$ given by all disagreeing positions, $\ones{y'}=\zeros{y'}+\ell$. Hence, $\ell$ indicates the surplus of ones generated by crossover. 
We describe a way of generating a random crossover output $y$ under the condition $A_\ell$. 
We first copy all bits that agree in $x_1$ and $x_2$ to $y$ and distribute $H(x_1,x_2)/2$ many ones and zeros uniformly at random among the positions in which $x_1$ and $x_2$ disagree. That gives an outcome $y^b$ of balanced uniform crossover.
Then we flip $\ell$ zeros in $y^b$ if $\ell \ge 0$ ($|\ell|$ ones if $\ell<0$), again chosen uniformly at random \newedit{among the disagreeing positions of $x_1$ and $x_2$}.
This yields a uniform random offspring $y$ amongst all offspring with a surplus of $\ell$ ones. Let $Q_t:=(P_t \cup \{y\}) \setminus \{x_d\}$. 
\begin{claim}
\label{lem:Qt-relation-to-Qtb}
    We have that 
    \begin{align}
    \label{Eq:ones-zeros}
    \E(S(Q_t) \mid A_\ell) \ge \E(S(Q_t^b)) -2\vert{\ell}\vert(\mu-1).
    \end{align}
\end{claim}

\begin{proofofclaim}
Let $Y(\ell,y^b)$ denote the set of bit strings obtainable from $y^b$ by flipping $\ell$ zeros. For $y \in Y(\ell,y^b)$ let $\Gamma_y:=\{i \in [n] \mid y_i \neq y_i^b\}$. Then we obtain with Lemma~\ref{lem:diversity-general2} conditioning on the event that fixed $y^b$ and $y \in Y(\ell,y^b)$ are produced
\begin{align*}
\notag \E(S(Q_t) \neweditx{\, \mid y^b,y}) &= \sum_{i \notin \Gamma_y} \E(S_i(Q_t) \neweditx{\, \mid y^b,y }) + \sum_{i \in \Gamma_y} \E(S_i(Q_t) \neweditx{ \, \mid y^b,y}) \\
&\notag \geq \sum_{i \notin \Gamma_y} \E(S_i(Q_t^b) \neweditx{\, \mid y^b,y}) + \sum_{i \in \Gamma_y} \E(S_i(Q_t^b) \neweditx{\, \mid y^b,y}) -2\vert{\ell}\vert(\mu-1)\\
&= \E(S(Q_t^b) \neweditx{\, \mid y^b,y}) -2\vert{\ell}\vert(\mu-1).
\end{align*}
Hence, we obtain with the law of total probability by summing over all possible $y_b$ and $y \in Y(\ell,y^b)$
\begin{align*}
\E(S(Q_t) \mid A_\ell) &\ge \E(S(Q_t^b)) -2\vert{\ell}\vert(\mu-1).
\end{align*}
\neweditx{This completes the proof of Claim~\ref{lem:Qt-relation-to-Qtb}.}
\end{proofofclaim}

We return to the proof of Theorem~\ref{thm:crossover-specific-parents-fixed}. For $i \in \{0, \ldots ,2k\}$ and $j \in \{0, \ldots,n-2k\}$ let $B_{ij}$ be the event of flipping $i$ zeros and $j$ ones in $y$ with the standard mutation operator. 
Let $q_{i,j,\ell}:=\pr(B_{ij} \mid A_\ell)$. Then for $i \in \neweditx{\{0, \dots, k-\ell\}}$ and $j \in \neweditx{\{0, \dots, n-k+\ell\}}$
\[
q_{i,j,\ell} = \binom{k-\ell}{i}\binom{n-k+\ell}{j} \left(\frac{\chi}{n}\right)^{i+j}\left(1-\frac{\chi}{n}\right)^{n-i-j}
\]
and $q_{i,j,\ell}=0$ otherwise. Note that $z$ has $k$ zeros (i.e. is in \plateau) if and only if $A_\ell$ and then $B_{i,i+\ell}$ occurs for some $i \in \{0, \ldots, k-\ell\}$ if $\ell \ge 0$, or $A_\ell$ and then $B_{i-\ell,i}$ occurs for some $i \in \{0, \ldots, k\}$ if $\ell < 0$.
If $A_\ell$ occurred, we have for $0 \leq i \leq k-\ell-1$ and $0 \leq j \leq n-k+\ell-1$
\begin{align}
\label{eq:Indices-q}
\notag q_{i+1,j+1, \ell} &= \binom{k-\ell}{i+1}\binom{n-k+\ell}{j+1} \left(\frac{\chi}{n}\right)^{i+j+2}\left(1-\frac{\chi}{n}\right)^{n-i-j-2}\\
&\notag =\left(\frac{\chi}{n-\chi}\right)^2\frac{k-\ell-i}{i+1} \cdot \frac{n-k+\ell-j}{j+1} \cdot q_{i,j,\ell}\\
&= O(k/n) \cdot q_{i,j,\ell}
\end{align}
due to $\binom{r}{s} = \frac{r-s+1}{s}\binom{r}{s-1}$ for $r,s \in \mathbb{N}$ with $r \geq s \geq 1$.

We continue the proof by estimating $\E(S(P_{t+1}) \mid A_\ell, B_{i,i+\ell})$ and $\E(S(P_{t+1}) \mid A_{-\ell},B_{i+\ell,i})$ for $\ell \geq 0$, encapsulated in the following nested lemma.

\begin{claim}
\label{claim-2}
\newedit{For $\ell \geq 0$} in the context of Theorem~\ref{thm:crossover-specific-parents-fixed},
\begin{align}
\label{eq:Expectation-estimation1-lemmarep}
& \E(S(P_{t+1}) \mid A_\ell,B_{i,i+\ell})  \geq  \E(S(Q^b_t)) -(i+\ell)\frac{4k\mu}{n-2k}
\end{align}
and
\begin{align}
\label{eq:Expectation-estimation2-lemmarep}
\E(S(P_{t+1}) \mid A_{-\ell},B_{i \neweditx{+} \ell,i}) \geq \E(S(Q^b_t)) -4\ell\mu - \frac{4ki\mu}{n-2k}.
\end{align}
\end{claim}
\begin{proofofclaim}
\neweditx{We start with~\eqref{eq:Expectation-estimation1-lemmarep}.}
Let $y$ be a possible outcome of crossover on $x_1,x_2$ with a surplus of $\ell$ ones. We will use that 
\begin{align*}
 \E(S(P_{t+1}) \mid y,B_{i,i+\ell}) &= \E(S(P_t \setminus \{x_d\}) \mid y,B_{i,i+\ell}) +2 \E(S_{P_t\setminus\{x_d\}}(z) \mid y,B_{i,i+\ell}) \\
&= \E(S(Q_t) \mid y) - 2\E(S_{P_t\setminus\{x_d\}}(y)) +2\E(S_{P_t\setminus\{x_d\}}(z) \mid y,B_{i,i+\ell}).
\end{align*}
So we need to estimate the difference $\Delta := S_{P_t\setminus\{x_d\}}(z)- S_{P_t\setminus\{x_d\}}(y)$ when we obtain $z$ from $y$ by flipping $i$ zeros and $i+\ell$ ones. Flipping a zero into a one can contribute at most $\mu-1$ to $\Delta$. Thus, the total contribution of $i$ flipping zeros to $\Delta$ is at least $-2i(\mu-1)$. \neweditx{Denote by $O_t \subset [n]$ the positions $i$ in which $y_i = 1$ and note that $|O_t| = n-k+\ell \ge n-2k$.} To estimate the contribution of flipping a one into zero, 
%
consider a uniformly random one in $y$ \neweditx{at position $j \in O_t$. Flipping this one contributes $\mu-1-2N_j$ to $\Delta$ where $N_j$ counts the number of zeros in $P_t \setminus \{x_d\}$ at position $j$, because each of the $N_j$ zeros contributes $-1$ and each of the $\mu-1-N_j$ ones contributes $+1$. Conditional on $B_{i, i+\ell}$, that is, flipping $i+\ell$ ones, by symmetry each one flips with probability $(i+\ell)/|O_t| \geq (i+\ell)/(n-2k)$, and so we obtain by summing over the ones in $y$ and using linearity of expectation (observing that $\sum_{j \in O_t} N_j \leq k \mu$ since the total number of zeros in $P_t\setminus\{x_d\}$ over \emph{all} positions is $(\mu-1) k$)}

\begin{align*}
& \E(S(P_{t+1}) \mid y,B_{i,i+\ell}) \\
&=(S(Q_t) \mid y)+2\E(\Delta \mid y, B_{i,i+\ell})\\
&\ge \E(S(Q_t) \mid y) -2i(\mu-1) + \neweditx{\frac{2(i+\ell)}{|O_t|} \cdot \E\Big(\sum_{j \in O_t}(\mu-1-2N_j)\Big)}\\
&\ge \neweditx{\E(S(Q_t) \mid y) -2i(\mu-1) + 2(i+\ell)\Bigl((\mu-1) -\frac{1}{|O_t|}\E\Big(\sum_{j \in O_t}2N_j\Big)\Bigr)}\\
&\geq  \E(S(Q_t)) \mid y) -2i(\mu-1) + 2(i+\ell)\left(\mu-1-\frac{2k\mu}{n-2k}\right).
\end{align*}
Since $\pr(B_{i,i+\ell} \mid y) = \pr(B_{i,i+\ell} \mid A_\ell)$ (for every possible outcome $y$ with $\ones{y}=\zeros{y}+\ell$ on disagreeing positions of $x_1,x_2$ the event $B_{i,i+\ell}$ occurs with the same probability) we obtain by the law of total probability and Inequality~\eqref{Eq:ones-zeros} 
\begin{align}
\label{eq:Expectation-estimation1}
\E(S(P_{t+1}) \mid A_\ell,B_{i,i+\ell}) 
&\notag \geq \E(S(Q_t) \mid A_\ell) -2i(\mu-1) + 2(i+\ell)\left(\mu-1-\frac{2k\mu}{n-2k}\right)\\
&\notag \geq  \E(S(Q^b_t)) -2(i+\ell)(\mu-1) + 2(i+\ell)\left(\mu-1-\frac{2k\mu}{n-2k}\right)\\
&= \E(S(Q^b_t)) -(i+\ell)\left(\frac{4k\mu}{n-2k}\right).
\end{align}
Now we estimate $\E(S(P_{t+1}) \mid A_{-\ell}, B_{i+\ell,i})$ for $\ell > 0$. Let $y$ be a possible outcome of crossover with $\ones{y'}=\zeros{y'}-\ell$. By an analogous argument as above flipping $i+\ell$ zero bits decreases the diversity by $2(i+\ell)(\mu-1)$ in $S(P_{t+1})$ compared to $S(Q_t \mid y)$ and flipping $i$ one\newedit{-}bits increases the diversity by $2i(\mu-1-\frac{2k\mu}{n-2k})$ in expectation. Hence,
\begin{align*}
\E(S(P_{t+1}) \mid y,B_{i+\ell,i}) &\notag \geq \E(S(Q_t) \mid y) -2(i+\ell)(\mu-1) + 2i\left(\mu-1-\frac{2k\mu}{n-2k}\right)\\
&= \E(S(Q_t) \mid y) -2\ell(\mu-1) - \frac{4ki\mu}{n-2k}, 
\end{align*}
and similar as above
\begin{align}\label{eq:Expectation-estimation2}
\notag \E(S(P_{t+1}) \mid A_{-\ell},B_{i-\ell,i}) &
\geq \E(S(Q_t) \mid A_{-\ell}) -2\ell(\mu-1) - \frac{4ki\mu}{n-2k}\\
&\geq \E(S(Q^b_t)) -4\ell(\mu-1) - \frac{4ki\mu}{n-2k}.
\end{align}
The statement \neweditx{of Claim~\ref{claim-2}} now follows by bounding $-4\ell(\mu-1) \ge -4\ell\mu$.
\end{proofofclaim}

The proof of Theorem~\ref{thm:crossover-specific-parents-fixed} continues. Now we estimate $\E(S(P_{t+1}))$ by $S(P_t)$. Denote $\rho_\ell := \pr(A_\ell)$. By the law of total expectation and using~\eqref{eq:Expectation-estimation1-lemmarep} (for $\ell \ge 0$) and~\eqref{eq:Expectation-estimation2-lemmarep} (for $\ell < 0$), we have 
\begin{align}
\label{eq:Inequality-Expectation-Total}
\E(S(P_{t+1}) \mid V)\pr(V) 
&\notag= \sum\nolimits_{\ell=0}^k \sum\nolimits_{i=0}^{k-\ell}\E(S(P_{t+1}) \mid A_\ell,B_{i,i+\ell})\rho_\ell q_{i,i+\ell,\ell}\\ 
&\notag+ \sum\nolimits_{\ell=1}^{k} \sum\nolimits_{i=0}^{k}\E(S(P_{t+1}) \mid A_{-\ell},B_{i+\ell,i})\rho_{-\ell}q_{i+\ell,i,-\ell}\\
&\geq \E(S(Q^b_t))\pr(V) - T_1 - T_2,
\end{align}
where 
\begin{align*}
&T_1:=\frac{4k\mu}{n-2k}\sum\nolimits_{\ell=0}^k\sum\nolimits_{i=0}^{k-\ell}(i+\ell)\rho_\ell q_{i,i+\ell,\ell} \text{ and }\\
&T_2:= \sum\nolimits_{\ell=1}^k\sum\nolimits_{i=0}^k\left(4\ell\mu+\frac{4k i \mu}{n-2k}\right)\rho_{-\ell}q_{i+\ell,i,-\ell}.
\end{align*}
\begin{claim}
\label{lem:tech-estimate}
    \neweditx{We have that $T_1,T_2 \leq 4k\mu\chi(1+O(k/n))/n$.}
\end{claim}
\begin{proofofclaim}
To bound $T_1$ and $T_2$ we use $\rho_\ell \leq 1$ and $q_{0,\ell,\ell} \leq \binom{n}{\ell}(\chi/n)^\ell(1-\chi/n)^{n-\ell}$. Moreover, note that $\sum\nolimits_{\ell \ge 0} q_{0,\ell,\ell} \le 1$ and $\sum\nolimits_{\ell \ge 0} \ell q_{0,\ell,\ell} \le \chi$. Hence, with Inequality~\eqref{eq:Indices-q}, 
\begin{align*}
    T_1 &\leq \frac{4k\mu}{n-2k}\left(\sum\nolimits_{\ell=0}^k \ell q_{0,\ell,\ell} + \sum\nolimits_{\ell=0}^k\sum\nolimits_{i=1}^{k-\ell}(i+\ell)q_{i,\ell+i,\ell}\right) \\
    &\leq \frac{4k\mu}{n-2k}\left(\sum\nolimits_{\ell=0}^k \ell q_{0,\ell,\ell} + \sum\nolimits_{\ell=0}^k\sum\nolimits_{i=1}^{k}(i+\ell)q_{0,\ell,\ell}O(k^i/n^i)\right) \\
    &\neweditx{\leq} \frac{4k\mu}{n-2k}\left(\chi + \sum\nolimits_{i=1}^{k}(i+\chi)O(k^i/n^i)\right) \\
    &= \frac{4k\mu \chi}{n-2k}\left(1 + \sum\nolimits_{i=1}^{k}(1+i/\chi)O(k^i/n^i)\right) \\
    &= \frac{4k\mu\chi}{n}(1+O(k/n)).
\end{align*}
Note also that $q_{\ell,0,-\ell} \leq \binom{k}{\ell}(\chi/n)^{\ell}(1-\chi/n)^{k-\ell}$\neweditx{,} which implies $\sum\nolimits_{\ell \geq 0}q_{\ell,0,-\ell} \leq 1$ and $\sum\nolimits_{\ell \geq 0}\ell q_{\ell,0,-\ell} \leq k\chi/n$
and thus again with Inequality~\eqref{eq:Indices-q}
\begin{align*}
    T_2 &\leq \sum\nolimits_{\ell=1}^k\sum\nolimits_{i=0}^{k}\left(4\ell\mu + \frac{4ki \mu}{n-2k}\right)q_{i+\ell,i,-\ell}\\
    &\leq \sum\nolimits_{\ell=1}^k4\ell\mu q_{\ell,0,-\ell} + \sum\nolimits_{\ell=1}^k\sum\nolimits_{i=1}^{k}\left(4\ell\mu + \frac{4ki \mu}{n-2k}\right)q_{i+\ell,i,-\ell}\\
    &\leq \frac{4k\chi\mu}{n} + \sum\nolimits_{\ell=1}^{k}\sum\nolimits_{i=1}^{k}\left(4\ell\mu+\frac{4ki\mu}{n-2k}\right)q_{\ell,0,-\ell}O(k^i/n^i)\\
    &\leq \frac{4k\chi\mu}{n} + \sum\nolimits_{i=1}^{k}\left(\frac{4k\chi\mu}{n}+\frac{4ki\mu}{n-2k}\right)O(k^i/n^i)\\
    &\leq \frac{4k\chi\mu}{n} \left(1 + \sum\nolimits_{i=1}^{k}\left(1+\frac{i}{\chi(1-2k/n)}\right)O(k^i/n^i)\right)\\
    &\leq \frac{4k\mu \chi}{n}(1+O(k/n))
\end{align*}
since the latter sum converges owing to $k/n \leq 1/4$. \neweditx{This completes the proof of Claim~\ref{lem:tech-estimate}.}
\end{proofofclaim}

Dividing~\eqref{eq:Inequality-Expectation-Total} by $\pr(V)$, plugging in the bound for $T_1$ and $T_2$, and increasing the constant from $4$ to $4.5$ in $T_1$ and $T_2$ to swallow \neweditx{the $(1+O(k/n))$ factor in $T_1$ and $T_2$} concludes the proof of Theorem~\ref{thm:crossover-specific-parents-fixed}.
\end{proof}

\subsection{Analysing the Equilibrium for the Population Diversity of the \mulgavar on Plateau}
\label{sec:equilibrium-state-crossover}

\neweditx{Continuing with our roadmap, we now use the diversity estimate from Theorem~\ref{thm:crossover-specific-parents-fixed} to analyse the dynamics of $S(P_t)$ for}
the \mulgavar stated in Algorithm \ref{alg:steady-state-GA}, mostly in the following situation.
\begin{situation}\label{sit:mulga}
Let $k \in o(n)$, $\lambdac \ge 6\sqrt{k}e^{\chi}\ln(\mu)$, and $\chi=\Theta(1)$. Consider the \mulgavar as in Algorithm~\ref{alg:steady-state-GA}, using standard bit mutation with mutation probability $\chi/n$ and uniform crossover with crossover probability $p_c$ on $\jump_k'$, and suppose the current population is $P_t \subset \plateau$.
\end{situation}

The following lemma is crucial for our analysis. It bounds the expected diversity in the next population in a generation with crossover. The following lower bound on $\E(S(P_{t+1}))$ shows that the current diversity $S(P_t)$ is diminished by a factor of $1-3/\mu^2$. In addition, we subtract an error term of $-9k\mu\chi/n$.
\begin{lemma}
\label{lem:crossover-specific}
In Situation~\ref{sit:mulga}, consider one generation $\neweditx{t}$ of the \mulgavar \newedit{when crossover occurs}.
Then
\begin{align*}
\E(S(P_{t+1})) \geq \big(1-\tfrac{3}{\mu^2}\big)S(P_t) - 9k\mu \chi/n.
\end{align*}
\end{lemma}

\begin{proof}
Let $x_1,x_2$ be the two parents. 
    By the law of total probability we obtain
\begin{align}
\label{eq:total-probability}
\E(S(P_{t+1})) = \frac{1}{\mu^2}\sum\nolimits_{(x_1,x_2) \in P_t^2} \E(S(P_{t+1}) \mid x_1,x_2).
\end{align}
    Let $V$ be the event that a fixed offspring $z$ is in \plateau conditioned on the choice of the parents, and let $V_\lambdac$ be the event that at least one out of the $\lambdac$ offspring is in \plateau. Note that $\E(S(P_{t+1}) \mid x_1,x_2,V) = \E(S(P_{t+1}) \mid x_1,x_2,V_\lambdac)$. Then again with the law of total probability for two parents $x_1,x_2$ and Theorem~\ref{thm:crossover-specific-parents-fixed}
\begin{align*}
    \notag\E(S(P_{t+1}) \mid x_1,x_2)&\geq \E(S(P_{t+1}) \mid x_1,x_2,V_\lambdac)\Pr(V_\lambdac \mid x_1,x_2)\\ 
    &\notag \geq \E(S(Q_t^b) \mid x_1,x_2)\Pr(V_\lambdac \mid x_1,x_2)-\tfrac{9k\mu \chi}{n},
\end{align*}
\newedit{where $Q_t$ and $Q_t^b$ are defined as in Lemma~\ref{lem:crossover-general} and Theorem~\ref{thm:crossover-specific-parents-fixed}.}
Note that the probability of $V$ is at least 
the probability that crossover generates an offspring in \plateau and mutation does not change it, hence $\Pr(V) \ge \binom{2k}{k}2^{-2k}(1-\chi/n)^n \geq (1-\chi/n)^n/(2\sqrt{k}) \geq e^{-\chi}/(3\sqrt{k})$. So we obtain for $\lambdac \geq 6\sqrt{k}e^{\chi} \ln(\mu)$ due to $1-x \leq e^{-x}$,
\begin{align*}
    \Pr(V_\lambdac \mid x_1,x_2) &= 1-\left(1-\pr(V)\right)^\lambdac 
     \ge  1-\left(1-e^{-\chi}/(3\sqrt{k})\right)^{6\sqrt{k}e^{\chi}\ln(\mu)} \\
    &\geq 1-e^{-2\ln(\mu)} = 1-1/\mu^2.
\end{align*}
So we obtain for every $(x_1,x_2) \in P_t^2$
\begin{align}
\label{eq:estimation-expectation}
\E(S(P_{t+1}) \mid x_1,x_2) & \geq \big(1-\tfrac{1}{\mu^2}\big)\E(S(Q_t^b) \mid x_1,x_2)-\tfrac{9k\mu \chi}{n}.
\end{align}
Since \emph{balanced uniform crossover} is diversity-neutral by~\cite[Theorem~26]{Lengler2024}, we obtain with Lemma~\ref{lem:crossover-general}
\begin{align} 
\label{eq:balanced-crossover}
\frac{1}{\mu^2} \sum\nolimits_{(x_1,x_2) \in P_t^2}\E(S(Q_t^b) \mid x_1,x_2) = \E(\neweditx{S(Q_t^b)}) = (1-\tfrac{2}{\mu^2})S(P_t).
\end{align}
Plugging~\eqref{eq:estimation-expectation} and~\eqref{eq:balanced-crossover} into~\eqref{eq:total-probability} and using $(1-1/\mu^2)(1-2/\mu^2) \ge (1-3/\mu^2)$ gives the lemma.
\end{proof}

The proof of Lemma~\ref{lem:crossover-specific} exploits the $\lambdac$ competing crossover offspring and the assumption $\lambdac \ge 6\sqrt{k}e^{\chi}\ln(\mu)$ in order to amplify the probability $\Prob(V_{\lambdac})$ of generating an offspring on the plateau. We remark that this is necessary to obtain the statement from Lemma~\ref{lem:crossover-specific} as the following lemma indicates. A complete proof is found in the appendix.
\begin{lemma}
\label{lem:negative-standard-GA}
    Consider a crossover step in the original \muga (or, equivalently, the \mulgavar with $\lambdac=1$) with \neweditx{$\chi=O(1)$}, a population $P_t:=\{x_1, \ldots ,x_\mu\}$ of size $\mu$ (w.\,l.\,o.\,g.\ a multiple of 4) and $\mu/4$ individuals $x_1, \ldots , x_{\mu/4}$ of the form $0^{\sqrt{k}}1^{n-k}0^{k-\sqrt{k}}$, and $3\mu/4$ individuals $x_{\mu/4+1}, \ldots , x_\mu$ of the form $1^{\sqrt{k}}0^{\sqrt{k}}1^{n-k-\sqrt{k}}0^{k-\sqrt{k}}$. Then $S(P_t) = 3\sqrt{k}\mu^2/4$ and for $k \in \omega(1)$
    \[
    \E(S(P_{t+1})) \leq S(P_t) - (1-o(1))\cdot \neweditx{3\sqrt{k} e^{-\chi} \mu/32}.
    \]
\end{lemma}

Note that in Lemma~\ref{lem:crossover-specific} the loss of diversity is at most $\frac{3}{\mu^2}S(P_t) + 9k\mu \chi/n= 9\sqrt{k}/4+9k\mu \chi/n$ in expectation which is roughly by a factor of at least $\Omega(\min\{\mu,n/\sqrt{k}\}) = \Omega(n/\sqrt{k})$ smaller than the loss of diversity stated in Lemma~\ref{lem:negative-standard-GA} if $\mu=\Omega(n)$ and $k=\omega(1)$. This is why our analysis is not easily extendable to the \muga. We suspect that this would require more fine-grained analysis of the population than the diversity measure $S(P_t)$ provides. \smallskip


Now we are able to derive a lower bound for $\E(S(P_{t+1}) \mid S(P_t))$ for the \mulgavar.

\begin{corollary}
\label{cor:diversity-GA-lower-bound}
In Situation~\ref{sit:mulga},
let 
\begin{align}\label{eq:alpha}
\alpha:=(1-p_c)4\mu \left(\sum\nolimits_{\ell=0}^k \ell p_\ell\right) - 9p_ck\mu  \chi /n
\end{align}
and
\begin{align}\label{eq:delta}
\delta:=\frac{1}{\mu^2}\left(3p_c + (2-2p_c)\sum\nolimits_{\ell=0}^k p_\ell \left(1 + \frac{(\mu-1)\ell n}{ k(n-k)} \right) \right)
\end{align}
where $p_\ell$ denotes the probability that \newedit{$\mutation$} flips $\ell$ ones and $\ell$ zeros. Then
\begin{align}\label{eq:lower-bound-SPt}
\E(S(P_{t+1}) \mid S(P_t))  \geq (1-\delta)S(P_t) + \alpha.
\end{align}
\end{corollary}

\begin{proof}
Denote by $A$ the event that crossover is executed \neweditx{when generating $P_{t+1}$} in the \mulgavar and by $B$ the event that no crossover is executed. Then we obtain by the law of total probability 
\begin{align}
\label{eq:ExpectationCrossover}
\E(S(P_{t+1}) \mid S(P_t)) = p_c\E(S(P_{t+1}) \mid S(P_t),A) + (1-p_c) \E(S(P_{t+1}) \mid S(P_t),B).
\end{align}
By~\eqref{eq:expectation} above we obtain 
\begin{align}
\label{eq:Crossovernotexecuted}
\E(S(P_{t+1}) \mid S(P_t),B) = (1-\gamma)S(P_t) + \beta
\end{align}
where $\beta$ and $\gamma$ are defined as in~\eqref{eq:betagamma} and with Lemma~\ref{lem:crossover-specific}
\begin{align}
\label{eq:Crossoverexecuted}
\E(S(P_{t+1}) \mid S(P_t),A) \geq \left(1-\frac{3}{\mu^2}\right) S(P_t) - 9k\mu \chi /n.
\end{align}
Plugging \eqref{eq:Crossovernotexecuted} and \eqref{eq:Crossoverexecuted} into \eqref{eq:ExpectationCrossover} gives the result.
\end{proof}

In order to provide intuition, we will use the term equilibrium state as a state in which the expected change of diversity is zero. \newedit{Note that our proofs do not rely on the existence of such an equilibrium state. For a finite state space the drift is trivially negative for large enough $S(P_t)$, in particular when $S(P_t)$ is at its maximum value. However, in general, there may be no equilibrium state due to rounding issues, or there may be more than one such state due to non-monotonicity of the drift.} 

\newedit{For our later paramater choices we will always have $\alpha >0$.}
Equation~\eqref{eq:lower-bound-SPt} implies that the equilibrium state of $S(P_t)$ is at least $\alpha/\delta$, since for $S(P_t) \le \alpha/\delta$ the drift is non-negative due to $\E(S(P_{t+1})-S(P_t)) \mid S(P_t))  \geq -\delta S(P_t) + \alpha \ge 0$. In the next \neweditx{corollary} we will compute how large $\alpha/\delta$ is for standard bit mutations with $\chi = \Theta(1)$. If $\mu=\Omega(n)$ then \newedit{we} will show that  $\alpha/\delta \ge (1-O(k^{3/2}p_c+n/\mu))2k\mu^2$. Note that $2k\mu^2$ is a deterministic upper bound for $S(P_t)$ because any two individuals in \plateau have distance at most $2k$. Hence, if $k^{3/2}p_c$ and $n/\mu$ are both small, then the following \newedit{corollary} implies that the equilibrium state is very close to the maximum.

For brevity, we write 
\begin{align}\label{eq:C1C2}
C_1 := \sum\nolimits_{\ell=0}^k p_\ell \quad\text{ and } \quad C_2 := \frac{n}{k}\sum\nolimits_{\ell=0}^k\ell p_\ell.
\end{align}
Note that for standard bit mutations with $\chi=\Theta(1)$ and for $k\le n/2$ we have $C_1,C_2 = \Theta(1)$ by Lemma~\ref{lem:Estimate-p-binomial}.

\begin{corollary}\label{cor:alphadelta}
In Situation~\ref{sit:mulga}, 
let $\alpha,\delta$ be as in~\eqref{eq:alpha},~\eqref{eq:delta} and $C_1,C_2$ as in~\eqref{eq:C1C2}. Then
\begin{align}
\label{eq:equilibrium}
\frac{\alpha}{\delta} \ge 2k\mu^2\cdot \left(1-p_c\frac{C_2+9\chi/4}{C_2}- \frac{k}{n-k}- \frac{n}{\mu-1}\frac{2C_1 + 3p_c}{2C_2}\right).
\end{align}
In particular, let $0<\eps<1$. If $p_c \le \frac{\eps}{3}\frac{C_2}{C_2+9\chi/4}$ and $\mu \ge 1 + \neweditx{\frac{3n}{\eps}\frac{2C_1+3p_c}{2C_2}}$ and $k \le \frac{\neweditx{(n-k)\eps}}{3}$ then $1/\delta = O(\mu n)$ and
\begin{align}
\label{eq:equilibrium2}
\alpha/\delta \ge 2k\mu^2\cdot \left(1-\eps\right).
\end{align}
\end{corollary}

\begin{proof}
We may write
\begin{align*}
    \alpha 
    = 4C_2\mu\frac{k}{n}\left(1-p_c(1+9\chi/(4C_2))\right).
\end{align*}
For $\delta$, using $1-p_c\le 1$ we obtain the upper bound
\begin{align}\label{eq:upperbounddelta}
\delta \leq \frac{1}{\mu^2}\left(3p_c + 2C_1 + \frac{2 C_2(\mu-1)}{n-k} \right) = \frac{2C_2(\mu-1)}{\mu^2n}\left(1+\frac{k}{n-k} + \frac{(2C_1 + 3p_c)n}{2C_2(\mu-1)} \right).
\end{align}
\neweditx{Hence, we can upper bound $\alpha/\delta$ as}
\begin{align*}
\frac{\alpha}{\delta} \ge 2k\mu^2\cdot \frac{1-p_c(1+9\chi/(4C_2))}{1+\frac{k}{n-k} +\frac{(2C_1 + 3p_c)n}{2C_2(\mu-1)}}.
\end{align*}
The bound~\eqref{eq:equilibrium} now follows from the inequality $(1-x)/(1+y) \ge (1-x)(1-y) \ge 1-x-y$ for all $x,y\ge 0$. Equation~\eqref{eq:equilibrium2} then follows immediately by plugging in \neweditx{under the additional assumption on $\mu$}. \neweditx{Further, we see with Lemma~\ref{lem:Estimate-p-binomial} that 
$$\delta \geq \frac{(\mu-1)n\sum_{\ell=0}^k \ell p_\ell}{\mu^2k(n-k)} \geq \frac{(\mu-1)\chi^2 e^{-\chi}}{\mu^2n} = \Omega\left(\frac{1}{\mu n}\right)$$
which shows $1/\delta = O(\mu n)$.}
\end{proof}

Corollaries~\ref{cor:diversity-GA-lower-bound} and~\ref{cor:alphadelta} imply for the \mulgavar that the expected change of the diversity $S(P_t)$ is positive unless $S(P_t)$ is already very close to its maximum $2k\mu^2$. Based on Lemma~\ref{lem:negative-standard-GA}, we show that this is not always true for the \mulga. For this algorithm there are states with much smaller diversity in which the expected change is negative, even for very small $p_c$. The proof is given in the appendix.
\begin{lemma}
\label{lem:negative-standard-GA-2}
    Consider the original \muga (or, equivalently, the \mulgavar with $\lambdac=1$) with \neweditx{$\chi=O(1)$}, a population $P_t:=\{x_1, \ldots ,x_\mu\}$ as in Lemma~\ref{lem:negative-standard-GA}, i.e. $\mu/4$ individuals $x_1, \ldots , x_{\mu/4}$ of the form $0^{\sqrt{k}}1^{n-k}0^{k-\sqrt{k}}$, and $3\mu/4$ individuals $x_{\mu/4+1}, \ldots , x_\mu$ of the form $1^{\sqrt{k}}0^{\sqrt{k}}1^{n-k-\sqrt{k}}0^{k-\sqrt{k}}$. Assume $k \in \omega(1)$. Then $S(P_t) = 3\sqrt{k}\mu^2/4$ and
    \[
    \E(S(P_{t+1})) - S(P_t) \leq -\neweditx{(1-o(1))3p_c\sqrt{k} e^{-\chi}\mu/32}+(1-p_c)4k \chi \mu/n.
    \]
    The right hand side is negative for $\neweditx{p_c > 128 \sqrt{k} \chi e^{\chi}/(3n(1-o(1))) = \Omega(\sqrt{k}/n)}$.  
   \end{lemma}

However, note that by Corollaries~\ref{cor:diversity-GA-lower-bound} and~\ref{cor:alphadelta} the \mulgavar exhibits a positive drift on the diversity of the population from Lemma~\ref{lem:negative-standard-GA-2} for certain $\mu = \Omega(n)$ and $p_c=O(1)$ if the requirements from Situation~\ref{sit:mulga} are met since $S(P_t) = 3\sqrt{k}\mu^2/4 = o(k\mu^2)$ for $k=\omega(1)$.


\subsection{Maintaining Diversity Close to the Maximum}
\label{sec:diversity-close-maximum}

We now provide a lemma which shows that if $\alpha/\delta$ is close to $2k\mu^2$ then after some short initial time $\tau_0$ it is likely \newedit{that} during the following $T$ generations the population diversity is also close to $2k\mu^2$ for at least $T/4$ generations. Here $T$ is an integer that can be chosen arbitrarily. This lemma presents a novel approach, compared to previous work. \Citet{Dang2017} showed that population diversity (measured with respect to the size of the largest species, that is, the largest number of copies of the same genotype in the population) is built up within a fixed period of time ($O(\mu n + \mu^2 \log \mu)$ generations) and then remains high for another period of time ($\Theta(\mu^2)$ generations) with good probability. \citet{Doerr2024} showed that diversity is maintained for a period of time that is exponential in~$\mu$. Our lemma offers a different approach: It shows that with good probability a constant fraction of generations has good diversity, during \textit{any} desired period of time~$T$.
Remarkably, this powerful statement is shown even though we only have statements on the \emph{expected} population diversity.
\begin{lemma}\label{lem:SPt-is-large}
In Situation~\ref{sit:mulga}, 
let $\alpha$ and $\delta$ as in~\eqref{eq:alpha} and~\eqref{eq:delta}, respectively. Suppose that $\alpha/\delta < 2k\mu^2$ and let $\eps \ge 1-\frac{\alpha}{2k\mu^2\delta} > 0$, \newedit{or equivalently}
\begin{align}\label{eq:eps}
\alpha/\delta \ge (1-\eps)2k\mu^2.
\end{align}
Let $\tau_0 := \lceil \ln(1/\eps)/\delta \rceil$, let $T\in \N$, and let $\mathcal{E}_T$ be the event that \newedit{$|\{\tau\in [T]:\ S(P_{t+\tau_0+\tau}) \ge (1-4\eps)2k\mu^2\}| \ge T/4$}. Then $\pr(\mathcal{E}_T) \ge 1/3$.
\end{lemma}
\begin{proof}
We start with the following lemma where we show that $\E(S(P_t))$ converges quickly to $\alpha/\delta$. Note that this lemma only examines the expectation $\E(S(P_t))$. In the subsequent proof, we will then deduce that similar statements also hold for the random variable $S(P_t)$.
\begin{lemma}\label{lem:convergence-of-expectation}
In Situation~\ref{sit:mulga}, let $\alpha$ and $\delta$ be as in~\eqref{eq:alpha} and~\eqref{eq:delta} respectively, and suppose that its current population is $P_t \subset \plateau$. Then for every $\tau \in \N$,
\begin{align}\label{eq:recursion-SPt}
\E(S(P_{t+\tau}) \mid S(P_t))  \geq (1-e^{-\delta\tau})\alpha/\delta.
\end{align}
\end{lemma}
\begin{proof}
Using~\eqref{eq:lower-bound-SPt} from Corollary~\ref{cor:diversity-GA-lower-bound}, we have 
\begin{align*}
\frac{\alpha}{\delta} - \E(S(P_{t+1}) \mid S(P_t)) \leq \frac{\alpha}{\delta} - (1-\delta)S(P_t)-\alpha = (1-\delta)\left(\frac{\alpha}{\delta}- S(P_t)\right).
\end{align*}
Hence, the random variable $Y_t := \alpha/\delta - S(P_{t})$ satisfies the recursion $\E(Y_{t+1} \mid Y_t) \le (1-\delta)Y_t$. Iterating this recursion $\tau$ times and using the law of total expectation, we obtain $\E(Y_{t+\tau}\mid Y_t) \le (1-\delta)^{\tau}Y_t$. Written out and using $S(P_t)\ge 0$, this is
\begin{align*}
\frac{\alpha}{\delta} - \E(S(P_{t+\tau}) \mid S(P_t)) \leq (1-\delta)^\tau\left(\frac{\alpha}{\delta}- S(P_t)\right) \le (1-\delta)^\tau\frac{\alpha}{\delta} \le e^{-\delta \tau}\frac{\alpha}{\delta}.
\end{align*}
This is equivalent to~\eqref{eq:recursion-SPt} and proves the lemma.
\end{proof}
Continuing the proof of Lemma~\ref{lem:SPt-is-large}, fix $\tau \in \N$. We first use Lemma~\ref{lem:convergence-of-expectation} for $\tau_0+\tau \ge \tau_0$ and find that
\begin{align}\label{eq:SPt-is-large-1}
\begin{split}
\E(S(P_{t+\tau_0+\tau}) \mid S(P_t)) & \geq (1-e^{-\delta(\tau_0+\tau)})\frac{\alpha}{\delta} \geq (1-e^{-\delta\tau_0})\frac{\alpha}{\delta} \\
& \ge (1-\eps)\frac{\alpha}{\delta} \ge (1-\eps)^2 2k\mu^2  \ge (1-2\eps)2k\mu^2.
\end{split}
\end{align}
Let us write $\rho:= \Pr(S(P_{t+\tau_0+\tau}) < (1-4\eps)2k\mu^2)$. 
Any two individuals in $P_{t+\tau_0+\tau}$ have Hamming distance $k$ from the optimum, so they have Hamming distance at most $2k$ from each other. 
Since $S(P_{t+\tau_0+\tau})$ is defined as the sum over $\mu^2$ such Hamming distances, we have $S(P_{t+\tau_0+\tau}) \le 2k\mu^2$. 
Therefore, we may bound
\begin{align*}
\E(S(P_{t+\tau_0+\tau}) \mid S(P_t)) \leq \rho\cdot (1-4\eps)2k\mu^2 + (1-\rho)\cdot 2k\mu^2 = 2k\mu^2 - 4\eps\cdot 2k\mu^2\cdot \rho.
\end{align*}
This implies
\begin{align}\label{eq:SPt-is-large-2}
\rho \le \frac{2k\mu^2 - \E(S(P_{t+\tau_0+\tau}) \mid S(P_t))}{4\eps\cdot 2k\mu^2} \stackrel{\eqref{eq:SPt-is-large-1}}{\le} \frac{2\eps\cdot 2k\mu^2}{4\eps\cdot 2k\mu^2} = 1/2.
\end{align}
Now consider the random variable 
\[
N_T := |\{\tau\in [T]:\ S(P_{t+\tau_0+\tau}) < (1-4\eps)2k\mu^2\}|.
\]
By definition of $\rho$ and by~\eqref{eq:SPt-is-large-2}, it has expectation $\E(N_T) \le T/2$. By Markov's inequality, we obtain $\pr(N_T > 3T/4) \le (T/2)/(3T/4) =  2/3$. On the other hand, the event $\mathcal{E}_T$ from the lemma is the complement of the event ``$N_T > 3T/4$'', and hence $\pr(\mathcal{E}_T) = 1- \pr(N_T > 3T/4) \ge 1/3$. This concludes the proof.
\end{proof}

\section{Runtime Bounds for the \protect\mulgavar on \jump}
\label{sec:Runtime-Analysis}

Now we exploit the lower bounds on population diversity to derive the desired runtime bounds of the \mulgavar on $\textsc{\jump}_k$. We use the fact that the population diversity approaches an equilibrium state to conclude that, in most generations, the population contains a good number of potential parents with a Hamming distance close to the equilibrium. In a generation with crossover, such good parents have a much improved chance to create the optimum via crossover and mutation, compared to using only mutation.

\neweditx{This section is structured as follows. In Section~\ref{sec:reach-plateau}, we first prove an upper bound on the expected time for the entire population to reach the plateau (Theorem~\ref{the:time-to-plateau}). In Section~\ref{sec:Runtime-Analysis-upper-bounds}, we show that our diversity estimates imply that, while the population evolves on the plateau, it typically contains many pairs of diverse parents (Lemma~\ref{lem:Cross-cheap}). We use this to derive general upper runtime bounds under loose conditions on $p_c$ and $k$ (Theorem~\ref{the:Runtime-1}), followed by stronger and simpler runtime bounds under tighter conditions on $p_c$ and~$k$ (Theorem~\ref{the:Runtime-2}). Finally, Section~\ref{sec:lower-bounds} provides matching lower bounds for classes of black-box algorithms (Theorem~\ref{the:lower-bound}).}


\subsection{Bounding the Expected Time to Reach the Plateau}
\label{sec:reach-plateau}

We first give an upper bound on the expected time for the whole population of the \mulgavar to reach \plateau. It is similar to 
Lemma~2.1 in~\cite{Dang2017}, but it includes evaluations for $\lambdac$ competing crossover offspring and allows arbitrary unbiased mutation operators. 
\begin{theorem}
\label{the:time-to-plateau}
    Consider the \mulgavar applying uniform crossover with crossover probability $p_c$ and an unbiased mutation operator flipping no bits with probability $f_0$ and flipping a single bit with probability $f_1$ on $\jump_k$. For all parameters $0 \le k \le n$, $p_c < 1$, $f_0 > 0$, $f_1 > 0$, $\mu, \lambdac \in \N$ and every initial population, the expected time for reaching the global optimum or creating a population on \plateau is 
    \[
     O\left(\left(1 + \frac{\lambdac p_c}{1-p_c}\right) \neweditx{\frac{1}{1-p_c}}\left(\mu n \left(\frac{1}{f_0} + \frac{1}{f_1}\right) + \frac{n \log n}{f_1} + \frac{\mu \log \mu}{f_0}\right)\right)\newedit{.}
    \]
     For standard bit mutation with mutation probability $\Theta(1/n)$ and $0<p_c \le 1 - \Omega(1)$ this simplifies to $O(\lceil \lambdac p_c\rceil (\mu n + n \log(n) + \mu \log \mu))$.
     Both bounds also apply to the \muga when setting $\lambdac \coloneqq 1$.
\end{theorem}
\begin{proof}
The proof is a generalisation of Witt's analysis of the \muea on \onemax~\cite{Witt2006}.


\neweditx{We first focus only on mutation-only steps, that is, generations that do not apply crossover, and later analyse the number of function evaluations performed in generations with crossover. More specifically, we will estimate the expected time to improve the best-so-far fitness and the expected time to generate a sufficient number of individuals with the current best fitness. Owing to elitism,  generations applying crossover cannot decrease the best-so-far fitness, nor can they reduce the number of search points that have a fitness of at least~$i$, for any fixed value~$i$. We denote by $M$ the random number of mutation-only steps performed until the entire population reaches the plateau, and bound $\E(M)$ from above without taking advantage of generations with crossover.}

Let $T_i$ denote the number of generations in which the best individual has a fitness of~$i$ and \neweditx{only mutation is performed. Let } $s_i$ denote the probability of mutation creating an offspring of fitness larger than~$i$ from a parent with fitness~$i$. 
Let $j \in \{1, \dots, \mu\}$. We first bound the expected \neweditx{number of mutation steps} until the population contains at least $j$ individuals of fitness $i$, or until it contains an individual of fitness larger than~$i$. Note that, if an individual of fitness~$i$ is cloned, an individual with smaller fitness is removed from the population. If there are $\ell$ individuals of fitness~$i$, a sufficient condition for the former event is not to execute crossover, to select one of $\ell$ fitness-$i$ individuals as parent and to not flip any bits. As these events are independent, the probability of these events is $(1-p_c) \cdot \neweditx{(\ell/\mu)} \cdot f_0$. The expected waiting time for increasing the number of individuals with fitness at least~$i$ is thus at most $\mu/((1-p_c)\ell f_0)$. Summing up all times for $\ell \in \{1, \dots, j-1\}$ yields an upper time bound of 
\begin{equation}
\label{eq:takeover-time}
    \frac{\mu}{(1-p_c)f_0} \cdot H(j-1) \le \frac{\mu}{(1-p_c)f_0} \cdot (\ln(j)+1)
\end{equation}
\newedit{where $H(j):=\sum_{\ell = 1}^j 1/\ell$}. Once the population contains at least $j$ individuals of fitness~$i$, the probability of creating an offspring with fitness larger than~$i$ is at least $(1-p_c) \cdot \neweditx{(j/\mu)} \cdot s_i$. The expected waiting time for this event is at most $\mu/(js_i(1-p_c))$. Together, we obtain the following family of upper bounds for all values of $j \in \{1, \dots, \mu\}$:
\[
    \E(T_i) \le 
    \frac{\mu}{1-p_c} \left(\frac{\ln(j)+1}{f_0} + \frac{1}{js_i}\right).
\]
Now we choose $\neweditx{j:=j_i} \coloneqq \min\{\lceil f_1/s_i \rceil, \mu\}$, then we can bound the expected time to first reach the plateau or the global optimum as follows:
\begin{align}
    & \sum\nolimits_{i=0}^{n-1} \E(T_i) \le \frac{\mu}{1-p_c} \sum\nolimits_{i=0}^{n-1} \left(\frac{\ln(j)+1}{f_0} + \frac{1}{js_i}\right)\notag\\
    \leq \;& \frac{\mu}{1-p_c} \left(\frac{n}{f_0} + \frac{1}{f_0}\sum\nolimits_{i=0}^{n-1} \ln(\lceil f_1/s_i \rceil) + \sum\nolimits_{i=0}^{n-1} \frac{1}{\min\{\lceil f_1/s_i \rceil, \mu\}s_i}\right).\label{eq:sum-of-E-Ti}
\end{align}
Note that for $0 \le i \le k-1$ we have $s_i \ge f_1(n-i)/n$ \newedit{since a parent with fitness $i$ has} $n-i$ ones and it suffices to flip a single one to obtain a fitness improvement. \neweditx{In this case $\max\{f_1/s_i,1\} \leq n/(n-i)$.} Similarly, for $k \le i \le n-1$ we have $s_i \ge f_1(n+k-i)/n$ as all search points \newedit{with fitness $i$} have $i-k$ ones and $n-(i-k)=n+k-i$ zeros and flipping a single zero is sufficient for a fitness improvement. \neweditx{In this case $\max\{f_1/s_i,1\} \leq n/(n+k-i)$.}
\neweditx{We obtain
\begin{align*}
\sum\nolimits_{i=0}^{n-1} \ln(\max\{1,f_1/s_i\}) &\le \sum\nolimits_{i=0}^{k-1} \ln \left(\frac{n}{n-i}\right) + \sum\nolimits_{i=k}^{n-1} \ln \left(\frac{n}{n+k-i}\right) \\
&\leq 2\sum\nolimits_{i=0}^{n-1} \ln \left(\frac{n}{n-i} \right).
\end{align*}
}
\neweditx{If $f_1/s_i \le 1$ then $\lceil f_1/s_i \rceil = 1$.} 
\neweditx{If $f_1/s_i > 1$ then $\lceil f_1/s_i\rceil \le 2f_1/s_i$. Hence, we obtain $\lceil f_1/s_i\rceil \leq \max\{1,2f_1/s_i\} \leq 2\max\{1,f_1/s_i\}$\neweditx{,} which implies}
\begin{align*}
    & \sum\nolimits_{i=0}^{n-1} \ln(\lceil f_1/s_i \rceil) \le \sum\nolimits_{i=0}^{n-1} \ln(\neweditx{2\max\{1,f_1/s_i\}}) \le \ln(2)n + \sum\nolimits_{i=0}^{n-1} \ln(\neweditx{\max\{1,f_1/s_i\}})\\
    \le\;& \ln(2)n + 2\sum\nolimits_{i=0}^{n-1} \ln\left(\frac{n}{n-i}\right) = \ln(2)n + 2 \ln\left(\prod_{i=0}^{n-1} \frac{n}{n-i}\right)\\
    =\;& \ln(2)n + 2 \ln\left(\frac{n^n}{n!}\right) \le \ln(2)n + 2 \ln\left(\frac{n^n}{(n/e)^n}\right) = \ln(2)n + 2 \ln(e^n) = (\ln(2)+2)n.
\end{align*}
The second sum is bounded as follows.
\begin{align*}
     & \sum\nolimits_{i=0}^{n-1} \frac{1}{\min\{\lceil f_1/s_i \rceil, \mu\}s_i}
     \le \sum\nolimits_{i=0}^{n-1} \frac{1}{\min\{f_1, \mu s_i\}}
     \le \sum\nolimits_{i=0}^{n-1} \left(\frac{1}{f_1} + \frac{1}{\mu s_i}\right)\\
     =\;& \frac{n}{f_1} + \frac{1}{\mu}\sum\nolimits_{i=0}^{n-1} \frac{1}{s_i} \le \frac{n}{f_1} + \frac{2}{f_1\mu}\sum\nolimits_{i=0}^{n-1} \frac{n}{n-i} = \frac{n}{f_1} + \frac{2nH(n)}{f_1\mu}. 
\end{align*}
Plugging this back into~\eqref{eq:sum-of-E-Ti}, we obtain an upper bound of 
\begin{align*}
    &\frac{\mu}{1-p_c} \left(\frac{n}{f_0} + \frac{(\ln(2)+2)n}{f_0} + \frac{n}{f_1} + \frac{2nH(n)}{f_1\mu}\right)\\
    &=\; O\left(\frac{\mu n}{1-p_c} \left(\frac{1}{f_0} + \frac{1}{f_1}\right) + \frac{n \log n}{(1-p_c)f_1}\right)
\end{align*}
for the expected time until the first point on the plateau (or a global optimum) is created, that is, a fitness of at least~$n$. Unless a global optimum is reached, we still need to account for the expected time until \emph{all} individuals have reached the plateau (or a global optimum is reached).
By~\eqref{eq:takeover-time}, the expected time until $j \coloneqq \mu$ individuals have fitness~$n$ (or a global optimum is reached) is at most 
\[
    \frac{\mu}{(1-p_c)f_0} \cdot (\ln(\mu)+1)
\]
and together this yields 
\neweditx{the following upper bound on the expected number of mutation-only steps:
\begin{equation}
    \label{eq:upper-bound-on-E-M}
    \E(M) = \frac{1}{1-p_c} \cdot O\left(\mu n \left(\frac{1}{f_0} + \frac{1}{f_1}\right) + \frac{n \log n}{f_1} + \frac{\mu \log \mu}{f_0}\right).
\end{equation}
Now we take into account the number of fitness evaluations performed in generations applying crossover. Recall that such generations perform $\lambdac$ fitness evaluations. Let} $L$ be the random number of evaluations made up to and including the first mutation-only step. Note that $L$ does not depend on the current population, and thus we have a sequence of iid random variables with the same distribution as $L$. Thus, we can upper bound the expected number of evaluations for the \mulga by $\E(M) \cdot \E(L)$ using Wald's equation. It remains to work out $\E(L)$. If the first mutation-only step happens in step~$t$, the \mulgavar makes $(t-1)\lambdac + 1$ evaluations. The probability of this happening in step~$t$ is $p_c^{t-1}(1-p_c)$. Together,
\begin{align*}
    \E(L) &= 1 + \sum\nolimits_{t=1}^\infty p_c^{t-1}(1-p_c) \cdot (t-1)\lambdac
    = 1 + (1-p_c)\lambdac\sum\nolimits_{t=0}^\infty p_c^{t} \cdot t\\
    &= 1 + (1-p_c)\lambdac \cdot \frac{p_c}{(1-p_c)^2} = 1 + \frac{p_c \lambdac}{1-p_c}.
\end{align*}
\neweditx{Multiplying this with the upper bound from~\eqref{eq:upper-bound-on-E-M} yields the claimed bound.}
The final statement follows easily using 
$1+x \le 2 = 2\lceil x\rceil$ if $0 < x \le 1$ and $1+x \le 2x \le 2\lceil x \rceil$ if $x > 1$, thus $1+x \le 2\lceil x \rceil$ and $1+ p_c \lambdac/(1-p_c) = O(\lceil \lambdac p_c/(1-p_c)\rceil) = O(\lceil \lambdac p_c\rceil)$ since $p_c = 1 - \Omega(1)$. \neweditx{In addition, the factor $1/(1-p_c)$ from~\eqref{eq:upper-bound-on-E-M} is in $O(1)$.}
\end{proof}

We remark that one can easily get a similar, but slightly weaker upper bound for $p_c=1$ 
because when selecting two identical parents 
for crossover, a standard bit mutation is performed. Since we always consider $p_c \le 1 - \Omega(1)$, we leave this for future work. The interested reader is referred to Lemma~1 in~\cite{Dang2017} for a similar  bound holding for the classical \muga with $p_c= \Omega(1)$, including $p_c=1$.

\subsection{Bounding the Expected Time on the Plateau}
\label{sec:Runtime-Analysis-upper-bounds}

Now we use our analysis of population diversity to bound the expected time to reach a global optimum, when starting from a population on the plateau.

The following lemma gives a lower bound on the Hamming distance of individuals if the population diversity is large. In this case, many pairs of large Hamming distances exist.
More specifically, we assume that $S(P_t) \ge 2k\mu^2(1-\eps)$ for a quantity $\varepsilon \coloneqq \varepsilon(k)$ that may depend on $k$. As the maximum possible diversity is $2k\mu(\mu-1)$, the diversity of $P_t$ is by a factor of roughly $1-\eps$ away from the maximum (ignoring the small difference between $\mu^2$ and $\mu(\mu-1)$).
\begin{lemma}
\label{lem:Cross-cheap}
Let $\eps \coloneqq \eps(k) >0$ and let $P_t \subset \plateau$ be a population of size~$\mu$. Suppose that $S(P_t)\ge 2k\mu^2(1-\eps)$. Then there are at least $\mu^2/2$ pairs $(x,y) \in P_t^2$ with $H(x,y) > (1-2\eps)2k$.
\end{lemma} 

\begin{proof}
Let $D$ be the set of pairs $(x,y) \in P_t^2$ with $H(x,y) > (1-2\eps)2k$. Note that $H(x,y)\le 2k$ since $x,y$ are in \plateau. Hence, each pair in $D$ has Hamming distance at most $2k$, and each pair in $P_t^2\setminus D$ has distance at most $(1-2\eps)2k$. Hence 
\[
S(P_t) \le |D|\cdot 2k + (\mu^2-|D|)(1-2\eps)2k = 2k\mu^2(1-2\eps) + 4\eps k|D|.
\]
Together with $S(P_t)\ge 2k\mu^2(1-\eps)$, this implies $2k\mu^2\eps \le 4\eps k|D|$, and thus $|D|\ge \mu^2/2$ as required.
\end{proof}

The next \newedit{theorem} gives an upper bound for the expected time until the \mulgavar (Algorithm \ref{alg:steady-state-GA}) finds the global optimum if we use our results from Lemma~\ref{lem:Cross-cheap} and Lemma~\ref{lem:SPt-is-large} in a general setting.

\begin{theorem}
\label{the:Runtime-1}
Consider the \mulgavar as in Algorithm~\ref{alg:steady-state-GA}, using standard bit mutation with mutation probability $\chi/n$ and uniform crossover with crossover probability $p_c$ on $\jump_k$.
Let $k \in o(n)$, $\lambdac \ge 6\sqrt{k}e^{\chi}\ln(\mu)$, $\chi=\Theta(1)$, and $0<\eps \coloneqq \eps(k) <1/4$. 
Suppose $p_c \le \frac{\eps}{3}\frac{2C_2}{2C_2+9\chi/4}$ and $\mu \ge 1 + \frac{3n}{\eps}\neweditx{\frac{2C_1+3p_c}{2C_2}}$ and $k\le \neweditx{(n-k)\eps}/3$, where $C_1,C_2$ are from~\eqref{eq:C1C2}. Then the global optimum is found in
\[
O \left(\lceil\lambdac p_c\rceil \mu(n\log(1/\eps) + \log(\mu)) + \lambdac + \frac{1}{p_c} + 4^{k-\lfloor 8\eps k\rfloor}\left(n/\chi\right)^{\lfloor8\eps k\rfloor} \left(1 + \frac{1}{p_c\lambdac}\right) \right)
\]
expected fitness evaluations.
If, additionally, $\lambdac = \Omega(1/p_c)$,
this simplifies to
\[
O \left(\lceil\lambdac p_c\rceil \mu(n\log(1/\eps) + \log(\mu)) + \lambdac + 4^{k-\lfloor 8\eps k\rfloor}\left(n/\chi\right)^{\lfloor8\eps k\rfloor}\right).
\]
\end{theorem}
Loosely speaking, the first term $\lceil\lambdac p_c\rceil \mu(n\log(1/\eps) + \log(\mu))$ bounds the expected time to reach a population on the plateau, and to build up sufficient diversity to facilitate a jump to the optimum. The remaining summands $\lambdac + 1/p_c + 4^{k-\lfloor 8\eps k\rfloor}\left(n/\chi\right)^{\lfloor8\eps k\rfloor} \left(1 + 1/(p_c\lambdac)\right)$ bound the expected time to create the global optimum from a pair of diverse parents.

Note that the parameter $\lambdac$ in the \mulgavar is reflected in the first term, as a factor of $\lceil \lambdac p_c\rceil$, since a large value of~$\lambdac$ may slow down the approach to the plateau. The bound contains a summand $+\lambdac$, which only dominates the runtime if $\lambdac$ is excessively large. For moderate or large values of~$k$ (and $\eps$ not too small), the term $4^{k-\lfloor 8\eps k\rfloor}\left(n/\chi\right)^{\lfloor8\eps k\rfloor} \left(1 + 1/(p_c\lambdac)\right)$ dominates the runtime bound. Here $\lambdac$ has no effect if the term $\left(1 + 1/(p_c\lambdac)\right)$ is dominated by $1$, that is, if $\lambdac = \Omega(1/p_c)$, and otherwise even has a positive effect on the runtime bound.
The reason large values of $\lambdac$ are not harmful is that, once the population is on the plateau and sufficiently diverse, generations executing crossover provide $\lambdac$ independent trials for a successful crossover. Although such a generation comes at a cost of $\lambdac$ fitness evaluations, the probability of achieving a successful crossover is amplified by a factor of roughly~$\lambdac$. These two effects balance each other out. If crossover is executed and sufficiently diverse parents are chosen, it is generally beneficial to use a large $\lambdac$ to amplify the probability of creating a global optimum. \begin{proof}[Proof of Theorem~\ref{the:Runtime-1}]
\newedit{Consider the algorithm on $\jump_k'$, so that the process continues after sampling the all-one-string.} According to Theorem~\ref{the:time-to-plateau}, we have that in $O(\lceil \lambdac p_c\rceil(\mu n + n\log(n)+\mu\log\mu))$ \neweditx{expected} fitness evaluations \newedit{$P_t\subset \plateau$}. Note that then all conditions of Situation~\ref{sit:mulga} are met. 
Since by assumption $\mu = \Omega(n)$, we have $n \log n \in O(\mu n)$ and the bound simplifies to a term absorbed by the stated bound:
\begin{equation}
    \label{eq:first-term-in-upper-bound}
    O(\lceil \lambdac p_c\rceil(\mu n + \mu \log \mu)) = O(\lceil\lambdac p_c\rceil \mu(n\log(1/\eps) + \log(\mu)))
\end{equation}
since $\eps < 1/4$ and thus $\log(1/\eps) \ge 2$.
The idea behind analysing the expected remaining time to find the global optimum is as follows. The remaining optimisation time is divided into so-called \emph{trials} of length $\tau_0 + T$, where $\tau_0 := \lceil \ln(1/\eps)/\delta\rceil$ is chosen according to Lemma~\ref{lem:SPt-is-large}. Intuitively, $\tau_0$ reflects an initial time to build up diversity in the population. The remaining time $T$ is chosen in such a way that the algorithm has sufficient chances during the $\tau_0 + T$ generations to generate the optimum, assuming the population is sufficiently diverse. We will choose $T$ in such a way that the probability of creating the optimum during a trial is $\Omega(1)$. If a trial is unsuccessful, we consider another trial.
Thus, in expectation, $O(1)$ trials are sufficient and the expected number of \emph{generations}, starting with a population on the plateau, is $O(\tau_0 + T)$. 
To bound the number of \emph{function evaluations}, we argue as follows. In every generation, the \mulgavar performs $1+p_c(\lambdac-1) \le 1 + p_c\lambdac$ fitness evaluations in expectation as it performs one evaluation for sure and a generation with crossover performs $\lambdac-1$ additional evaluations. Thus, the expected number of function evaluations, starting with a population on the plateau, is 
\begin{equation}
    \label{eq:abstract-upper-bound}
    O((\tau_0 + T)(1+p_c\lambdac)) = O(\lceil \lambdac p_c\rceil \tau_0 + (1+p_c\lambdac)T)
\end{equation}
where in the first term we used $1+p_c \lambdac \le 2\lceil \lambdac p_c\rceil$.

We now simplify this first summand in~\eqref{eq:abstract-upper-bound}. Recall $\tau_0 := \lceil \ln(1/\eps)/\delta\rceil$  and note that by Corollary~\ref{cor:alphadelta} we have $1/\delta = O(\mu n)$, which implies $\tau_0= O(\mu n \log(1/\eps))$. Together, we have 
\[
    \lceil \lambdac p_c \rceil \tau_0 = O(\lceil \lambdac p_c\rceil \mu n \log(1/\eps)),
\]
which is absorbed by the first term of the claimed upper bound~\eqref{eq:first-term-in-upper-bound}.

It remains to define $T$ in such a way that the probability of a successful trial is $\Omega(1)$, and then to bound the remaining term in~\eqref{eq:abstract-upper-bound}, $(1+p_c\lambdac)T$, from above. To this end, we first consider a bound $\popt$ on the probability of crossover creating the optimum if crossover is executed and sufficiently diverse parents are chosen. Building up to invoke Lemma~\ref{lem:SPt-is-large},  by Corollary~\ref{cor:alphadelta} we have $\alpha/\delta \ge (1-\eps)2k\mu^2$ \newedit{as in~\eqref{eq:eps}}. Let $\mathcal{E}$ be the event that during the trial we have $S(P_t) \ge (1-4\eps)2k\mu^2$ for at least $T/4$ generations in the trial. We call such generations \emph{good} generations. Then $\Pr(\mathcal E)\ge 1/3$ by Lemma~\ref{lem:SPt-is-large}. 

By Lemma~\ref{lem:Cross-cheap}, each good generation has at least $\mu^2/2$ pairs $(x_1,x_2)\in P_t^2$ with distance $H(x_1,x_2) > 2k(1-8\eps)$ from each other, which we call \emph{good pairs}. This implies that the number $r$ of positions in which both $x_1$ and $x_2$ have a zero-bit is less than $8\eps k$. Since $r$ is integral, we even have $r \le \lfloor8\eps k\rfloor$. In a good generation, a good pair is picked for crossover with probability at least $p_c/2$. 
We call this a \emph{good-crossover batch} as the algorithm produces a batch of $\lambdac$ offspring via crossing the same good parents. 
With each offspring in such a batch it may create the optimum by (i) setting all $2(k-r)$ differing bit positions to~$1$ via crossover, (ii) correcting the remaining $r$ bits via mutation, and (iii) not flipping any of the correct bits with mutation. The probability of each of these steps is $2^{-2(k-r)} = 4^{r-k}$, $(\chi/n)^r$, and $(1-\chi/n)^{n-r} = \Theta(1)$, respectively. As all these events are independent and $r \le \lfloor8\eps k\rfloor$, the probability of a fixed offspring in a good-crossover batch being the optimum is at least
\[
    \popt \coloneqq 4^{\lfloor8\eps k\rfloor-k} \cdot (\chi/n)^{\lfloor8\eps k\rfloor} \cdot (1-\chi/n)^{n-\lfloor8\eps k\rfloor}.
\]
Now we use $\popt$ to define the time period
\[
    T \coloneqq \left\lceil \frac{\newedity{16}}{p_c} \cdot \max\left\{\frac{1}{\lambdac \popt}, \newedity{8}\right\}\right\rceil.
\]
%

\newedity{Denote by $G$ the number of good-crossover batches during a trial, and let $\mathcal A$ be the event ``$G \ge p_c T/16$.'' We claim that $\Pr(\mathcal E \cap (\neg \mathcal A)) \le 1/6$. To see this, consider an auxiliary process in which the duration of the trial is made infinite, i.e., we just continue the trial beyond the first $\tau+T$ generations. Let $G'$ be the number of good-crossover batches among the first $T/4$ good generations of the auxiliary process. Note that these exist because we made the process infinitely long and Lemma~\ref{lem:SPt-is-large} can be applied repeatedly.}

\newedity{Recall that a good generation begins with a population of high diversity that contains at least $\mu^2/2$ good pairs. The decision to perform crossover, as well as the selection of parents, is independent of decisions made in other generations. Therefore, in every good generation, a good crossover batch occurs independently with probability at least $p_c/2$. This implies that $G'$ stochastically dominates a binomial random variable~$X$ with parameters $T/4$ and $p_c/2$. In particular, $\E[X] = p_c T/8 \ge 16$, owing to the maximum in the definition of~$T$. 
Moreover, we have $(\mathcal{E} \cap (G < p_c T/16)) \subseteq (\mathcal{E} \cap (G' < p_c T/16))$ since $\mathcal{E}$ implies that $G' \le G$.
Therefore,
\begin{align*}
\Pr(\mathcal E \cap (\neg \mathcal A)) & = \Pr(\mathcal E \cap (G < p_c T/16))  \\ & \le \Pr(\mathcal E \cap (G' < p_cT/16)) \\ & 
\le \Pr(G' < p_cT/16) \\ & \le \Pr(X < p_cT/16) \le e^{-\E[X]/8} \le e^{-2} \le \frac{1}{6},
\end{align*}
where the first step in the last line follows from $p_c T/16 = \E[X]/2$ and the Chernoff bound~\cite[(1.10.12)]{doerr-theory-chapter}. In particular,
\begin{align*}
\Pr(\neg( \mathcal E \cap \mathcal A)) \le \Pr(\neg \mathcal E) + \Pr(\mathcal E \cap (\neg \mathcal A)) \le \frac{2}{3} + \frac{1}{6} = \frac{5}{6},
\end{align*}
and thus $\Pr(\mathcal E \cap \mathcal A) \ge 1/6$.
}


\newedity{If $\mathcal A$ occurs then} $G \ge \newedity{p_cT/16} \ge 1/(\lambdac \popt)$, \newedity{hence} at least $\lambdac G \ge 1/\popt$ offspring are created in good-crossover batches. Then the optimum is created during the trial with probability at least 
\[
    1-(1-\popt)^{1/\popt} \ge 1-\frac{1}{e}.
\]
\newedity{To conclude, let $\mathcal B$ be the event that the optimum is created during the trial. Then
\begin{align*}
    \Pr(\mathcal B) \ge \Pr(\mathcal E \cap \mathcal A \cap \mathcal B) & = \Pr(\mathcal B \mid \mathcal E \cap \mathcal A)\cdot \Pr(\mathcal E \cap \mathcal A) \ge (1-1/e)/6.
\end{align*}}
Hence, each trial is successful with probability $\Omega(1)$ and the term $(1+p_c\lambdac)T$ contributing to the time bound is at most
\begin{align*}
    (1+p_c\lambdac)\left(1 + \frac{\newedity{16}}{p_c\lambdac \popt} + \frac{\newedity{128}}{p_c}\right)
    =\;& O\left(\lambdac + \frac{1}{p_c} + \frac{1}{\popt} \left(1 + \frac{1}{p_c\lambdac}\right)\right).
\end{align*}
Plugging in $\popt = \Theta(4^{\lfloor8\eps k\rfloor-k} \cdot (\chi/n)^{\lfloor8\eps k\rfloor})$ completes the proof of the first bound.

The second bound follows easily from the first one. Since $\lambdac = \Omega(1/p_c)$, the term $+1/p_c$ is absorbed by the $+\lambdac$ term and the term $(1+1/(p_c\lambdac))$ disappears since it simplifies to a factor of $O(1)$.
\end{proof}

This general bound from Theorem~\ref{the:Runtime-1} can be rewritten as follows if we set $\eps \coloneqq \eps(k) = 1/(16k)$. With an $\eps$ this small, the possibly dominating term $n^{\lfloor 8\eps k\rfloor}$ becomes $n^{\lfloor 1/2 \rfloor} = n^0 = 1$ and so disappears completely. The only catch is that this only works for a restricted range of $k$, since we need $k \le \neweditx{(n-k)\eps/3}$ or, equivalently, $k^2 \le (n-k)/48$, which in particular implies $k = O(\sqrt{n})$.

\begin{theorem}
\label{the:Runtime-2}
Consider the \mulgavar as in Algorithm~\ref{alg:steady-state-GA}, using standard bit mutation with mutation probability $\chi/n$ and uniform crossover with crossover probability $p_c$ on $\jump_k$.
If $\lambdac \ge 6\sqrt{k}e^{\chi}\ln(\mu)$ and $\lambdac = \Omega(1/p_c)$, $\chi=\Theta(1)$, $p_c \le \frac{1}{48k}\frac{2C_2}{2C_2+9\chi/4}$, $\mu \ge 1 + 48kn\neweditx{\frac{2C_1+3p_c}{2C_2}}$ and $k^2\le (n-k)/48$ where $C_1,C_2$ are from~\eqref{eq:C1C2} the global optimum is found in
\begin{align}
\label{eq:simplified-bound-with-lambda-pc}
O \big(\lceil \lambdac p_c\rceil \mu(n\log(k) + \log(\mu)) + \lambdac + 4^k \big)
\end{align}
expected fitness evaluations. If we fix $\lambdac \coloneqq \lceil \max\{6\sqrt{k}e^{\chi}\ln(\mu), 1/p_c\}\rceil$ and also have the weak requirements $p_c = \Omega(4^{-k})$ 
and $\mu \le 2^{n\log k}$,
this upper bound simplifies to
\begin{align}\label{eq:final_result}
O \big(\mu n\log(\mu) + 4^k \big)
=
\begin{cases}
O \big(\mu n\log(\mu)\big) & \text{if $k \le \log(\mu n \log(\mu))/2$}\\
        O\big(4^k\big) & \text{if $k \ge \log(\mu n \log(\mu))/2$.}
    \end{cases}
\end{align}
\end{theorem}

\begin{proof}
For $\eps = 1/(16k)$ the first statement follows immediately from Theorem~\ref{the:Runtime-1} since $\lfloor8\eps k\rfloor = 0$ and thus $4^{k-\lfloor 8\eps k\rfloor}\left(n/\chi\right)^{\lfloor8\eps k\rfloor} = 4^k$.
The second statement exploits $\lceil \lambdac p_c\rceil \le 1 + O(\sqrt{k}\log(\mu) \cdot p_c) \le 1 + O(\log(\mu)/\sqrt{k})$, which follows from the condition $p_c = O(1/k)$. Along with $\log(\mu) \le n\log k$, the first summand in~\eqref{eq:simplified-bound-with-lambda-pc} simplifies to
\begin{align*}
    \mu \lceil \lambdac p_c\rceil(n\log(k) + \log(\mu))
    \le\;& \mu (1 + O(\log(\mu)/\sqrt{k})\neweditx{)} \cdot 2n\log(k)\\
    =\;& 2\mu n\log(k) + O(\mu n \log(\mu)\log(k)/\sqrt{k})\\
    =\;& O(\mu n \log(\mu))
\end{align*}
where the last step used $\log(k)/\sqrt{k} = O(1)$ and $k \le \mu$.
The additional $+\lambdac$ term is absorbed as by assumption $\lambdac = O(\sqrt{k}\log(\mu) + 1/p_c) = O(\sqrt{k}\log (\mu) + 4^k)$ using the assumption $p_c = \Omega(4^{-k})$. 

The right-hand side of~\eqref{eq:final_result} follows by observing that 
\[
    k \le \log(\mu n \log(\mu))/2 \Leftrightarrow 4^k \le \mu n \log(\mu)
\]
and, depending on $k$, the respective smaller term can be dropped.
\end{proof}

Theorem~\ref{the:Runtime-2} shows that, for small jump sizes $k$, $k \le \log(\mu n \log(\mu))/2$, the bound on the expected optimisation time is $O(\mu n \log(\mu))$. This indirectly depends on $k$ since we require $\mu = \Omega(kn)$. However, for the smallest value of~$\mu$ satisfying the conditions of Theorem~\ref{the:Runtime-2} we have $\mu = \Theta(kn)$ and then $O(\mu n \log(\mu)) = O(kn^2 \log(kn)) = O(n^2\log^2 n)$. Now this bound is independent of~$k$ and it applies to all $k \le \log(\mu n \log(\mu))/2$. 
Here the (upper bound on the) expected time to create the optimum via a crossover of good parents is dominated by the (upper bound on the) expected time to reach the plateau and to build up population diversity. 
For larger~$k$, the converse is true and the expected time to jump to the optimum via crossover dominates the expected optimisation time.
Finally, note that, for constant~$k$, we obtain upper bounds of $O(kn^2 \log(kn)) = O(n^2 \log n)$.

\subsection{Lower Bounds}
\label{sec:lower-bounds}

\begin{algorithm2e}[ht]
\DontPrintSemicolon
  $t \gets 0$\;
  Let $H_0 \subset \plateau$ be arbitrary.\;
  \While{global optimum not found}{
    Decide whether to apply crossover\;
        \uIf{crossover is applied}{
            Select some $x_{1},x_{2}$ from $H_t$\;
            $y \gets \mathrm{crossover}(x_{1},x_{2})$\;
            $z \gets \mutation(y)$\;
        }
        \Else{
            Select some $y$ from $H_t$\;
            $z \gets \mutation(y)$\;
        }
        $f_t \gets \min\{f(x) \mid x \in H_t\}$\;
        \lIf{$f(z) \ge f_t$}{$H_{t+1} \gets H_t \cup \{z\}$}
    $t \gets t+1$
}
\caption{Scheme for elitist black-box algorithms on \plateau.}
\label{alg:black-box-general}
\end{algorithm2e}

\begin{algorithm2e}[ht]
\DontPrintSemicolon
  $t \gets 0$\;
  Let $H_0 \subset \plateau$ be arbitrary.\;
  \While{global optimum not found}{
    Choose $b\in [0,1)$ uniformly at random\;
    \For{$i=1$ \KwTo $\lambda$}{
        \uIf{$b < p_c$}{
            Select some $x_{i, 1},x_{i, 2}$ from $H_t$\;
            $y_i \gets \mathrm{crossover}(x_{i, 1},x_{i, 2})$\;
            $z_i \gets \mutation(y_i)$\;
        }
        \Else{
            Select some $y_i$ from $H_t$\;
            $z_i \gets \mutation(y_i)$\;
        }
        $f_t \gets \min\{f(x) \mid x \in H_t\}$\;
        $H_{t+1} \gets H_t \cup \{z_i \mid i \in [\lambda], f(z_i) \ge f_t\}$\;
    }
    $t \gets t+1$
}
\caption{Scheme of elitist black-box algorithms with crossover probability $p_c \in [0,1]$ and static offspring population size $\lambda$ on \plateau}
\label{alg:black-box-specific}
\end{algorithm2e}

The following theorem gives a lower bound for the expected time to discover the global optimum from the plateau. It is very general as it applies to  \neweditx{a large class of elitist} black-box algorithms using unary unbiased mutation operators~\cite{Lehre2012}. In brief, a black-box algorithm iteratively evaluates the fitness of a search point (the initial one is chosen uniformly at random) and then chooses the next search point to be evaluated based on the history of previously evaluated search points and their fitness. \neweditx{We assume that the algorithm starts with an arbitrary subset of $\plateau$ and that it only accepts search points with at least the same fitness. Note that elitist algorithms starting with a uniform random population may fit into the scheme of Algorithm~\ref{alg:black-box-general} once the whole population has reached the plateau. Instead of considering a population $P_t$, in Algorithm~\ref{alg:black-box-general} we consider a \emph{history} $H_t$ of unbounded size, from which parents can be selected. Algorithm~\ref{alg:black-box-general} shows a scheme of this class of algorithms that employs some unbiased mutation operator.} The class of unary unbiased mutation operators includes standard bit mutation, operators flipping a fixed number of bits, and heavy-tailed mutations~\cite{Doerr2017-fastGA} that choose the number of flipping bits from a heavy-tailed distribution such as a power-law distribution.

\neweditx{In particular, Algorithm~\ref{alg:black-box-general} includes the \mulgavar once the population is on the plateau. Note that parent selection and the decision whether to perform crossover in Algorithm~\ref{alg:black-box-general} can be implemented such that we obtain one generation of the \mulgavar; in case crossover is executed, it would spend the next $\lambdac$ steps generating offspring from the same parents.}

The theorem also gives \neweditx{a lower bound~\eqref{eq:lower-runtime-bound-2}} for ``conventional'' algorithms that execute crossover with probability $p_c$ and create the same number $\lambda$ of offspring, regardless of whether crossover is used or not\neweditx{, as shown in Algorithm~\ref{alg:black-box-specific}. Note that this algorithm is contained in the scheme of Algorithm~\ref{alg:black-box-general} as the latter can simulate the generation of a batch of $\lambda$ offspring in $\lambda$ subsequent steps}. Our \mulgavar intentionally deviates from \neweditx{the scheme of Algorithm~\ref{alg:black-box-specific}}, hence this lower bound does not apply to the \mulgavar.
\begin{theorem}
\label{the:lower-bound}
    Consider any \neweditx{elitist} black-box algorithm $\mathcal{A}$ on $\jump_k$ with $k \le \sqrt{n}/2$ \neweditx{fitting the scheme of Algorithm~\ref{alg:black-box-general}: it} creates offspring by either applying some unary unbiased variation operator to a solution from $\plateau$ or by applying uniform crossover to two parents from $\plateau$, followed by some unary unbiased mutation. Then there is a constant $C > 0$ such that the expected number of function evaluations to find the global optimum is at least $\min\{C \cdot 4^k, \binom{n}{k}\}$.

    Moreover, assume that $\mathcal{A}$ \neweditx{fits the more specific scheme of Algorithm~\ref{alg:black-box-specific}:} in each generation it independently decides whether to execute crossover with probability $p_c \in (0, 1)$ and then creates a batch of $\lambda$ offspring with or without crossover as above, respectively, where $\lambda \in \N$ is an arbitrary parameter. Then the expected  number of function evaluations to find the global optimum is at least
    \begin{align}\label{eq:lower-runtime-bound-2}
        \frac{1}{2} \cdot \min\left\{\frac{\binom{n}{k}}{1-p_c}, \frac{C \cdot 4^k}{p_c}\right\}
    \end{align}
\end{theorem}
In particular, the bound~\eqref{eq:lower-runtime-bound-2} in Theorem~\ref{the:lower-bound} applies to all \muga algorithms analysed on $\jump_k$ in previous work~\cite{Jansen2002,Koetzing2011a,Dang2017,Doerr2024} and yields a lower bound of $\Omega(4^k/p_c)$ for all $k \le \sqrt{n}/2$ and $p_c \ge (4k/n)^k$ (as then $4^k/p_c \le (n/k)^k \le \binom{n}{k}$). Note that~\eqref{eq:lower-runtime-bound-2} does not apply to the \mulgavar since this algorithm only creates one offspring instead of $\lambda$ offspring in case of mutation. 
Since the \mulgavar has an upper bound of $O(4^k)$ for all $k \ge \log(\mu n \log(\mu))/2$, Theorem~\ref{the:lower-bound} proves that it has an asymptotically optimal performance for such~$k$, and that it is provably faster than the standard \muga by a factor of at least $1/p_c$.
\begin{proof}[Proof of Theorem~\ref{the:lower-bound}]
    Consider a generation executing crossover on two individuals $x_1,x_2 \in P_t$ with Hamming distance $2d$ on the plateau. (The number must be even since the number of $0/1$-bits and $1/0$ bits in $x_1,x_2$ must be equal when both are in \plateau.) \newedit{Let} $y$ be an offspring created by crossover and standard bit mutation with mutation probability~$p_m$. Note that $2d \le 2k$ by the triangle inequality, because both parents have distance  $k$ from the optimum. 

Now, if crossover sets $i$ of the $2d$ bits to one in which the parents differ, the offspring has $k+d-i$ zeros. To create the optimum, all zeros and no ones must be flipped. It is well known (see, e.\,g\neweditx{.,}\ Lemma~1 in~\cite{Doerr2020}) that every unbiased mutation operator can be described as follows: pick a radius $r$ (deterministically or according to some probability distribution) and then flip $r$ different bits chosen uniformly at random. Assuming that the mutation operator picks radius $r=k+d-i$ (otherwise it is impossible to create the optimum), the probability of choosing precisely all zeros to be flipped is $1/\binom{n}{r} = 1/\binom{n}{k+d-i}$.
Thus, the probability of $y$ being the optimum is 
\begin{align}\label{eq:prob-of-hitting-opt}
\pr[y = \vec 1] \neweditx{\; \leq \;} \sum\nolimits_{i=0}^{2d} \binom{2d}{i}2^{-2d} \cdot \frac{1}{\binom{n}{k+d-i}} = 4^{-d} \sum\nolimits_{i=0}^{2d} \binom{2d}{i}/\binom{n}{k+d-i}.
\end{align}
(This generalises~\cite[Lemma~2]{Dang2017} to arbitrary unbiased mutation operators.)
Denoting the summands in~\eqref{eq:prob-of-hitting-opt} as $S_0, \dots, S_{2d}$, we have, for all $i \in \{1, \dots, 2d\}$,
\begin{align*}
    \frac{S_i}{S_{i-1}} = \frac{\binom{2d}{i}}{\binom{2d}{i-1}} \cdot \frac{\binom{n}{k+d-i+1}}{\binom{n}{k+d-i}}= \frac{2d-i+1}{i} \cdot \frac{n-k-d+i}{k+d-i+1}= \frac{2d-i+1}{k+d-i+1} \cdot \frac{n-k-d+i}{i}
\end{align*}
and we observe that both factors are non-increasing with~$i$. Thus, the minimum of $S_i/S_{i-1}$ is attained for $i=2d$ and we get
\begin{align*}
    \frac{S_i}{S_{i-1}} \ge\;& \frac{1}{k-d+1} \cdot \frac{n-k+d}{2d} \ge \frac{n-k}{2d(k-d+1)}.
\end{align*}
The denominator is maximised for $d = (k+1)/2$, which implies 
\begin{align*}
    \frac{S_i}{S_{i-1}} \ge\;& \frac{n-k}{(k+1)((k-1)/2+1)} = \frac{2(n-k)}{(k+1)^2} \ge 2
\end{align*}
by assumption on~$k$.
Hence, the whole sum is dominated by the term for $i=2d$ and 
\begin{align}
\label{eq:probability-of-y-being-optimum}
\pr[y = \vec 1] = O\left(4^{-d} /\binom{n}{k-d} \right),
\end{align}
which is maximised for $2d=2k$ where $\binom{n}{k-d}=1$. Therefore, in a crossover step followed by mutation, for any two parents on the plateau, the probability that the offspring is the optimum is $O(4^{-k})$.

In generations without crossover, the parent has $k$ zeros and the probability of creating the optimum by mutation is at most $1/\binom{n}{k}$. Since in every offspring creation the probability of creating the optimum is bounded by $\max\{4^{-k}/C, 1/\binom{n}{k}\}$, where $C >0$ is the reciprocal of the implicit constant in $O(4^{-d})$, the expected number of fitness evaluations to generate the optimum is at least $\min\{C \cdot 4^k, \binom{n}{k}\}$ as claimed.

The second statement follows from noticing that, taking a union bound over $\lambda$ offspring, the probability of one generation creating the optimum is at most
\[
    \lambda\left((1-p_c) \cdot \frac{1}{\binom{n}{k}} + p_c \cdot C \cdot 4^{-k}\right)
    \le 2\lambda \max\left\{(1-p_c) \cdot \frac{1}{\binom{n}{k}}, \ p_c \cdot C \cdot 4^{-k}\right\}
\]
and so the expected number of generations is at least
\[
    \frac{1}{2\lambda} \cdot \min\left\{\frac{\binom{n}{k}}{1-p_c}, \frac{C \cdot 4^k}{p_c}\right\}.
\]
Multiplying by $\lambda$ yields the claimed bound on the expected number of function evaluations.
\end{proof}

Note that the condition $k^2\le (n-k)/48$ in Theorem~\ref{the:Runtime-2} implies $k \le \sqrt{n}/2$ (as required by the lower bound in Theorem~\ref{the:lower-bound}) and $4^k \le \binom{n}{k}$. Thus, we conclude that in this regime the upper bound of Theorem~\ref{the:Runtime-2} is tight.
\begin{corollary}
    \label{cor:Corollary-4-to-k}
Consider the \mulgavar on $\jump_{k}$ with $k, \mu, \lambdac$ and $p_c$ meeting the conditions of the second statement of Theorem~\ref{the:Runtime-2}.
If $k \ge \log(\mu n \log(\mu))$ then the expected optimisation time on $\jump_{k}$ is $\Theta(4^k)$.
\end{corollary}

\section{Extensions to Other Functions}
\label{sec:extensions}

At first glance, it might seem that our diversity estimates and the resulting runtimes are highly specific to $\jump_k$. However, these insights can also be applied in other contexts. Consider a fitness function in which there is some search point $x^*$ such that all search points with the same Hamming distance $k$ from~$x^*$ have the same fitness. Let us denote this set by $S$. Once the \mulgavar has reached a population $P_t \subset S$, it will behave as on $\jump_k'$ until a search point outside of $S$ is accepted. The reason is that the \mulgavar is agnostic to the bit values 0 and 1, and thus it shows the same dynamic behaviour on $S$ (all search points with Hamming distance~$k$ to~$x^*$) as on $\plateau$ (all search points with Hamming distance~$k$ to $1^n$). Thus, the diversity shows the same dynamics in both scenarios up to the point where any search point $x \notin S$ is accepted. If no other search point has the same fitness as search points in $S$, this implies that $f(x) > f(s)$ for all $s \in S$ and a strict fitness improvement was found. We can then easily derive upper bounds on the expected time to find fitness improvements. This argument can possibly be iterated or be integrated in a custom analysis. 

\subsection{Generalised Jump Functions}

As an example, we consider the function $\jump_{k, \delta}$~\cite{BamburyBD24} that was proposed independently as \textsc{JumpOffset} in~\cite{Rajabi2024,WITT202318}. It is defined like $\jump_k$, however the fitness valley has a width of $\delta$. 
\[
    \jump_{k, \delta} \coloneqq
    \begin{cases}
        \ones{x} & \text{if $\ones{x} \in [0, \dots, n-k] \cup [n-k+\delta, \dots, n]$\neweditx{,}}\\
        -\ones{x} & \text{otherwise.}
    \end{cases}
\]
\neweditx{The original $\jump_k$ function is similar to the special case where $\delta = k$, except for the fact that in $\jump_k$ both cases differ by additive terms of $+k$ and $+n$, respectively. However, the relative order of fitness values for search points with less than $n-k+\delta$ ones is identical for $\jump_k$ and $\jump_{k,k}$. 
Formally, for all $i, j \in \{0, \dots, n-k+\delta-1\}$, all $x_i$ with $\ones{x_i} = i$, and all $x_j$ with $\ones{x_j} =j$, we have
\[  
    \jump_k(x_i) \le \jump_k(x_j) \Leftrightarrow \jump_{k,k}(x_i) \le \jump_{k,k}(x_j)    
\]
and the equivalence remains valid when replacing ``$\le$'' with ``$<$'', ``$=$'', ``$\ge$'', or ``$>$'' on both sides.
Thus, as}
long as no point with at least $n-k+\delta$ ones has been found, every algorithm \neweditx{selecting a new population purely based on comparisons of search points} behaves as on $\jump_k$. Consequently, our analysis applies. The probability of uniform crossover applied to parents of Hamming distance $2k$ yielding a search point with at least $n-k+\delta$ ones is 
\[
    \sum_{i=0}^{k-\delta} \binom{2k}{k+\delta+i} \cdot 4^{-k} \ge \binom{2k}{k+\delta} \cdot 4^{-k} = \frac{(2k)!}{(k+\delta)!(k-\delta)!} \cdot 4^{-k}.
\]
Using this probability instead of $4^{-k}$ in our previous analysis, we obtain the following result. 
In case of small gaps $\delta$, specifically $\delta = O(\sqrt{k})$, the above probability is $\Omega(1/\sqrt{k})$ (see, e.\,g.\ Corollary~1.4.12 in~\cite{doerr-theory-chapter}). 
Then the expected waiting time for a successful jump is $O(\sqrt{k})$, which vanishes in the $O(\mu n \log(\mu))$ term.  
Note that the expected time for the \mulgavar to reach the optimum after completing the jump is also bounded by $O(\mu n \log(\mu))$ since Theorem~\ref{the:time-to-plateau} also applies to $\jump_k$ with $k=0$, which is \onemax, and the plateau equals $\{1^n\}$. Hence, we obtain the following theorem.
\begin{theorem}
\label{the:Runtime-Jump-k-delta}
Consider the \mulgavar on $\jump_{k, \delta}$ with $k, \mu, \lambdac$ and $p_c$ meeting the conditions of the second statement of Theorem~\ref{the:Runtime-2}.
Then the expected optimisation time on $\jump_{k, \delta}$ is 
\begin{align}
\label{eq:bound-for-jump-k-delta}
O \left(\mu n\log(\mu) + 4^k \cdot \binom{2k}{k+\delta}^{-1} \right).
\end{align}
For $\delta = O(\sqrt{k})$ this simplifies to $O(\mu n \log \mu)$. 
\end{theorem}
This upper bound compares favourably against the expected optimisation time of the \EA on $\jump_{k, \delta}$~\cite{BamburyBD24} (for $\delta \ge 3$ and $k \le n - \omega(\sqrt{n})$), which equals
\[
    (1+o(1))\left(\frac{en}{\delta}\right)^\delta \binom{k}{\delta}^{-1}.
\]

\subsection{Hurdle Functions}

We further consider the \hurdle function class that forces EAs to perform multiple jumps. It was introduced
by Pr{\"u}gel-Bennett~\cite{PRUGELBENNETT2004135}
as an example
class where genetic algorithms with crossover
outperform hill climbers. It is also an example of a problem with a ``big valley structure'' found in many combinatorial optimisation problems~\cite{Ochoa2016,Reeves1999}.

\begin{figure}[htb]
\centering{}
\begin{tikzpicture}[scale=0.9]
    \begin{axis}[
        width=12cm,
        height=8cm,
        grid=both,
        xlabel={$\ones{x}$},
        ylabel={$\hurdle(x)$},
        xmin=0, xmax=20,
        ymin=-5, ymax=0,
        ytick distance=1,
        xtick distance=1,
    ]
    \addplot[
        blue,
        mark=*,
        thick,
        mark options={scale=0.8},
        samples at={0,1,...,20},
    ]
    expression{
        -ceil((20-x)/5) - (mod((20-x),5)/5)
    };
    \end{axis}
\end{tikzpicture}
\caption{Plot of the \hurdle function with hurdle width $w=5$ and $n=20$.}
\label{fig:hurdle}
\end{figure}

It comes with a parameter $w \coloneqq w(n) \in \N$ called the \emph{hurdle width} that determines the gap between neighbouring local optima with respect to the number of zeros in the bit string, denoted as $\neweditx{\zeros{x}}$. Now \hurdle is defined as 
$$
\hurdle(x) = - \left\lceil{\frac{\neweditx{\zeros{x}}}{w}}\right\rceil - \frac{\rem(\neweditx{\zeros{x}},w)}{w}.
$$
where $\rem(z(x),w)$ is the remainder of $z(x)$ divided by $w$. Figure~\ref{fig:hurdle} shows a sketch of the \hurdle function for hurdle width $w=5$, where the $x$ axis reflects the number of ones (instead of zeros) to show the similarity with the $\jump_k$ function.

It is evident from Figure~\ref{fig:hurdle} and it was formally shown in~\cite{Nguyen2019} that $1^n$ is the only global optimum, all search points with $iw$ zeros, $i \in \N$, are local optima, and all search points with at most $iw-w$ zeros have a strictly larger fitness than all search points with $iw$ zeros. 

\citet{Nguyen2019} gave an asymptotically tight upper bound on the expected optimisation time of the \EA on the \hurdle function of
\[
    O(n \log n) + \sum_{i=1}^{\lfloor n/w \rfloor} \frac{en^w}{i^w} = O(n^w).
\]
Here the term $\frac{en^w}{i^w}$ is an upper bound on the expected time to improve the fitness from a search point of Hamming distance $iw$ to the optimum. Note that these upper bounds decrease drastically with increasing Hamming distance, and that the expected time to overcome the last hurdle at Hamming distance $1 \cdot w$ dominates the expected optimisation time.

Note that, once the population only contains search points with exactly $iw$ zeros, for any value $i \in \N$, the algorithm temporarily behaves as on $\jump_{iw, w}$ since only search points with $iw$ or at most $iw-w$ zeros are accepted. Hence, we can easily re-use our previous bound for $\jump_{iw, w}$. The only catch is that Theorem~\ref{the:Runtime-2} requires $k^2\le (n-k)/48$. This is implied by $k \le \sqrt{n}/7$ if $n$ is large enough. Our general bound from Theorem~\ref{the:Runtime-1} makes looser assumptions and only requires $k = o(n)$, but it provides weaker guarantees on the population diversity. For local optima far from the global optimum, this is not a huge issue since even diversity far below the maximum value suffices to escape the local optimum. After all, at a Hamming distance $iw$ from the optimum, it suffices to create a search point at Hamming distance $(i-1)w$ by creating a surplus of $w$ ones during crossover. However, for the sake of simplicity, we restrict our attention to tackling hurdles at Hamming distance at most $\sqrt{n}/7$ to the optimum. As discussed for the \EA, these hurdles are by far the hardest obstacles and we conjecture that the expected time to overcome these last hurdles dominates the overall expected optimisation time.

Now, assuming the population consists of search points with $iw$ zeros, the expected time to improve the fitness on \hurdle equals the expected time to improve the fitness on $\jump_{k, \delta}$ with $k \coloneqq iw$ and $\delta \coloneqq w$. Theorem~\ref{the:Runtime-Jump-k-delta} yields an upper bound of
\[
    O\left(\mu n \log(\mu) + 4^{iw} \cdot \binom{2iw}{iw+w}^{-1}\right)
\]
for this expected time. Note that the term $O(\mu n \log \mu)$ also covers the expected time for reaching a population of local optima. 
For $i=1$ we have $4^{iw} \cdot \binom{2iw}{iw+w}^{-1} = 4^w$. For $i \ge 2$, 
using Corollary~1.4.11 in~\cite{doerr-theory-chapter}, we can simplify the binomial coefficient as follows.
\begin{align*}
    4^{iw} \binom{2iw}{iw+w}^{-1} =\;& 4^{iw} \cdot \Theta\left(\sqrt{\frac{(iw+w)(iw-w)}{2iw}}\right) \left(\frac{iw+w}{2iw}\right)^{iw+w} \left(\frac{iw-w}{2iw}\right)^{iw-w}\\
    =\;& \Theta\left(\sqrt{\frac{(iw+w)(iw-w)}{2iw}}\right) \left(\frac{i+1}{i}\right)^{iw+w} \left(\frac{i-1}{i}\right)^{iw-w}\\
    =\;& \Theta\left(\sqrt{iw}\right) \left(\left(\frac{i+1}{i}\right)^{i+1} \left(\frac{i-1}{i}\right)^{i-1}\right)^w.
\end{align*}
The term in brackets is simplified as
\begin{align*}
    \left(\frac{i+1}{i}\right)^{i+1} \left(\frac{i-1}{i}\right)^{i-1}
    =\;& \frac{i+1}{i-1} \cdot \left(\frac{(i+1)(i-1)}{i^2}\right)^{i}\\
    =\;& \frac{i+1}{i-1} \cdot \left(\frac{i^2-1}{i^2}\right)^{i}\\
    \le\;& \frac{i+1}{i-1} = 1 + \frac{2}{i-1}.
\end{align*}
Adding up $O(\mu n \log(\mu) + 4^w)$ for $i=1$ and $O(\mu n \log(\mu) + \Theta\left(\sqrt{iw}\right) (1+2/(i-1))^w)$ for all $i \in \{2, \dots, \lfloor \sqrt{n}/(7w)\rfloor\}$, we get a total upper bound of 
\begin{align*}
    & O\left(\frac{\mu n^{3/2} \log(\mu)}{w}  + 4^w + \sum_{i=2}^{\lfloor \sqrt{n}/(7w)\rfloor} \sqrt{iw} \left(1 + \frac{2}{i-1}\right)^w\right)\\
    =\;& O\left(\frac{\mu n^{3/2} \log(\mu)}{w}  + 4^w + \sum_{i=2}^{w} \sqrt{iw} \left(1 + \frac{2}{i-1}\right)^w + \sum_{i=w+1}^{\lfloor \sqrt{n}/(7w)\rfloor} \sqrt{iw} \left(1 + \frac{2}{i-1}\right)^w\right)\\
    =\;& O\left(\frac{\mu n^{3/2} \log(\mu)}{w}  + 4^w + \sum_{i=2}^{w} w \cdot 3^w + \sum_{i=w+1}^{\lfloor \sqrt{n}/(7w)\rfloor} n^{1/4} \left(1 + \frac{2}{w}\right)^w\right)\\
    =\;& O\left(\frac{\mu n^{3/2} \log(\mu)}{w}  + 4^w + w^2 \cdot 3^w + \frac{\sqrt{n}}{w} \cdot n^{1/4} \cdot e^2\right)\\
    =\;& O\left(\frac{\mu n^{3/2} \log(\mu)}{w}  + 4^w + w^2 \cdot 3^w + \frac{n^{3/4}}{w}\right).
\end{align*}
Note that the term $+n^{3/4}/w$ can be absorbed by $+\mu n^{3/2} \log(\mu)/w$ and further, since $w^2 \cdot 3^w = O(4^w)$, the term $w^2 \cdot 3^w$ is absorbed by $+4^w$.
Hence, we have shown the following. \begin{theorem}
    The expected optimisation time of the \mulgavar on \hurdle with hurdle width $w$, using parameters $\mu, \lambdac$ and $p_c$ as in the second statement of Theorem~\ref{the:Runtime-2}, when starting with a population of search points that all have at most $\sqrt{n}/7$ zeros, is at most
    \[
        O\left(\frac{\mu n^{3/2} \log(\mu)}{w} + 4^w\right).
    \]
\end{theorem}
Note that for large enough~$w$ the term $4^w$ dominates the above bound. This makes sense as, like for $\jump_k$, the final jump is the hardest one. Then the resulting time bound is much smaller than the expected runtime of $\Theta(n^w)$ proven for the \EA~\cite{Nguyen2019}.

\section{Conclusions}

We have presented the first runtime bounds for a \muga variant that show strong benefits of crossover in the regime of large crossover probabilities $p_c$ without explicit diversity-enhancing mechanisms. Our analysis has shown that the population evolves near-perfect diversity, where most pairs of parents have the maximum Hamming distance of~$2k$. 
We believe that our analysis can be easily transferred to a \mulgavar when omitting mutation after crossover as, unlike the work in~\cite{Dang2017}, the optimum can be reached without mutation. 



A limitation of our approach is that we had to modify the original \muga to include $\lambdac$ competing offspring. An important task for future work is therefore to show similar results for the original \muga. Another research question is whether the dependency on $p_c$ can be further relaxed, possibly up to $p_c=1$ as in~\cite{Dang2017,Doerr2024}. For this one might need to carefully consider the possible advantages that crossover can provide for population diversity.

We have also shown that our approach can be applied to other functions where we have plateaus of all search points with the same Hamming distance from some fixed search point~$x^*$. We have demonstrated this for 
the $\jump_{k, \delta}$ benchmark~\cite{BamburyBD24} 
that contains a fitness valley of a specified width~$\delta$, and for the function \textsc{Hurdle}~\cite{PRUGELBENNETT2004135} that contains many peaks.

We leave the following open problems and research questions for future work.
\begin{enumerate}
    \item Can our upper bound for the \mulgavar be transferred to the original \muga, when allowing for a factor of $1/p_c$ in the term $4^k$? We suspect that this might be true, but the examples in Lemmas~\ref{lem:negative-standard-GA} and~\ref{lem:negative-standard-GA-2} suggest that this would require more fine-grained information about the population $P_t$ than provided by the diversity $S(P_t)$. 
    \item Can the conditions on $p_c$ be relaxed? This could possibly be achieved with stronger guarantees on how generations with crossover affect the population diversity (see Lemma~\ref{lem:crossover-specific}).
    \item Our analysis requires a large population size of $\mu = \Omega(kn)$, although empirical results~\cite{Li2023,Dang2017} suggest that small populations are highly effective as well. Can our results be extended towards smaller population sizes?
    \item Related to this, can the upper bound $O(\mu n \log(\mu)) = O(n^2 \log^2 n)$ for $k = O(\log n)$ and $\mu = \Theta(kn)$ be improved further?
    \item Do our diversity estimates and runtime bounds translate to ($\mu$+$\lambda$)~GAs in which more than one new search point can enter the population in the same generation? The algorithm might still create a larger number $\lambdac \gg \lambda$ of offspring in generations with crossover.
\end{enumerate}

\subsection*{Acknowledgements}

The authors thank the participants of Dagstuhl seminar 24271 ``\emph{Theory of Randomized Optimization Heuristics}'' for fruitful discussions that helped to refine the results and improve the presentation of our work.

\bibliographystyle{abbrvnat}
\bibliography{references}

\appendix

\section{Details on Adapted Runtime Bounds in Table~\ref{tab:overview-runtime-results}}

In this short section we give further details on adaptations made to published bounds from~\cite{Koetzing2011a} and~\cite{Doerr2024} included in Table~\ref{tab:overview-runtime-results}.

\Citet{Koetzing2011a} showed a bound of 
\[
    O(\mu n \log(n) + e^{6k}  \mu^{k+2}  n)
\]
for $p_c \le k/n$ and for a mutation operator which never creates duplicates of the parent. Inspecting the proof of \cite[Theorem~7]{Koetzing2011a}, we found a minor bug. The probability of event E6 of crossover happening at least once during $n/k$ generations and no mutation leading to an accepted offspring is bounded from below by
\begin{equation*}
\label{eq:Koetzing-inequality}
    1 - (1-p_c)^{n/k} \cdot \left(1 - \frac{k}{n}\right)^{k/n} = \Omega(1).
\end{equation*}
There are two typos in the expression on the left-hand side; \neweditx{the statement} should read
\[
    (1 - (1-p_c)^{n/k}) \cdot \left(1 - \frac{k}{n}\right)^{n/k} \neweditx{\;= \Omega(1)}.
\]
\neweditx{This corrected statement holds} for $p_c = \Theta(k/n)$, but for smaller $p_c$ we only have $1-(1-p_c)^{n/k} = \Theta(p_c \cdot n/k)$ (cf.\ \cite[Lemma~10]{Badkobeh2015}). This leads to an additional factor of $1/(p_c \cdot n/k) = k/(np_c)$ in the asymptotic bound, which then becomes
\[
    O(\mu n \log(n) + e^{6k}  \mu^{k+2}  k/p_c).
\]

\Citet{Doerr2024} very recently proved an upper bound of 
\begin{align*}
    O\Big(& \mu n \log(\mu) + n \log(n) + \frac{(n/k+\mu)\log(\mu)}{(n/\chi)^{-k+1}\min(\exp(p_c\mu/(2048e)), (n/\chi)^{+k-1})} + \left(\frac{n}{\chi}\right)^{k-1}\Big)
\end{align*}
for $2 \le \mu = o(n)$.
The third summand can be written as
\begin{align*}
    \frac{(n/k+\mu)\log(\mu)}{\min((n/\chi)^{-k+1}\exp(p_c\mu/(2048e)), 1)} \le \frac{(n/k+\mu)\log(\mu)}{(n/\chi)^{-k+1}\exp(p_c\mu/(2048e))} + \frac{(n/k+\mu)\log(\mu)}{1}.
\end{align*}
Since $(n/k + \mu)\log(\mu) = O(\mu n \log(\mu))$, the last term can be omitted and we obtain a simplified bound of
\begin{align*}
    O\left( \mu n \log(\mu) + n \log(n) + \left(\frac{n}{\chi}\right)^{k-1}\left(1 + 
     \frac{(n/k+\mu)\log(\mu)}{\exp(p_c\mu/(2048e))}\right)\right).
\end{align*}

\section{Proofs for Negative Results on the Drift of the Original \protect\muga}

Here we give proofs for negative results on the drift of the population diversity in the original \muga that were omitted from the main part to keep the paper streamlined.

\subsection{Proof of Lemma~\ref{lem:negative-standard-GA}}

We prove Lemma~\ref{lem:negative-standard-GA}, which shows that our drift bounds in generations with crossover derived for the \mulgavar do not hold for the original \muga.

\begin{proof}
Note that $H(x_i,x_j)=2\sqrt{k}$ for $i \in \{1, \ldots , \mu/4\}$ and $j \in \{\mu/4+1, \ldots , \mu\}$ and $H(x_i,x_j)=0$ otherwise. Therefore \begin{align}\label{eq:counterexample_population_diversity}
S(P_t) = 4\sqrt{k} \cdot \mu/4 \cdot 3\mu/4 = 3\sqrt{k}\mu^2/4,
\end{align}which becomes $o(k\mu^2)$ if $k \in \omega(1)$. Note further $P_t \neweditx{\; \subseteq \;} \plateau$. Let $y$ be an offspring which is generated by uniform crossover of two parents chosen uniformly at random from $P_t$ and $z$ be the outcome of mutation on $y$. Let $E_1$ be the event that the two parents of $y$ are $0^{\sqrt{k}}1^{n-k}0^{k-\sqrt{k}}$, $E_2$ be the event that those are $1^{\sqrt{k}}0^{\sqrt{k}}1^{n-k-\sqrt{k}}0^{k-\sqrt{k}}$ and $E_3$ be the event that one parent is $0^{\sqrt{k}}1^{n-k}0^{k-\sqrt{k}}$ and the other one $1^{\sqrt{k}}0^{\sqrt{k}}1^{n-k-\sqrt{k}}0^{k-\sqrt{k}}$.  
Let $Q_t:=(P_t \setminus \{x_d\}) \cup \{y\}$ and $P_{t+1}=(P_t \setminus \{x_d\}) \cup \{z\}$ where $d \in [\mu]$ is chosen uniformly at random. Then by the law of total probability
\begin{align*}
\E(S(P_{t+1}) \mid S(P_t)) &= \sum_{i=1}^3\E(S(P_{t+1}) \mid S(P_t),E_i)\neweditx{\Pr}(E_i).
\end{align*}
Note that $\neweditx{\Pr}(E_1)=1/16$, $\neweditx{\Pr}(E_2)=9/16$ and $\neweditx{\Pr}(E_3)=6/16$ and if $E_1$ or $E_2$ occurs, crossover does not have any effect. 
Let $V$ be the event that $z$ is on \plateau and let $\overline{V}$ denote its complementary event. Note that the diversity does not change if $z$ is not on \plateau since it is discarded in this case. Therefore, again by the law of total probability for $i \in \{1,2,3\}$,
\begin{align*}
&\E(S(P_{t+1}) \mid S(P_t),E_i)\\
&= \E(S(P_{t+1}) \mid S(P_t),E_i,V)\Pr(V \mid E_i) + \E(S(P_{t+1}) \mid S(P_t),\overline{V})\Pr(\overline{V} \mid E_i)\\
&= \E(S(P_{t+1}) \mid S(P_t),V,E_i)\Pr(V \mid E_i) + S(P_t)(1-\Pr(V \mid E_i)).
\end{align*}
To simplify notation, we drop the conditions on the event $E_i$ in the following occurring in the expectation with respect to $S(P_{t+1})$, $S(Q_t)$ and $\Pr(V)$. Suppose that $E_i$ occurs for $i \in \{1,2\}$. For $j \in \{0, \ldots , k\}$ let $B_j$ be the event that $j$ ones and $j$ zeros are flipped and let $q_j:=\neweditx{\Pr}(B_j \mid E_i)$. Note that $z$ is on \plateau if and only if one of the $B_j$ occurs. Hence, we obtain again by the law of total probability
\begin{align*}
\E(S(P_{t+1}) \mid S(P_t)) &= S(P_t)(1-\Pr(V)) + \sum_{j=0}^{\neweditx{k}}\E(S(P_{t+1}) \mid S(P_t),B_j)q_j.
\end{align*}
Further,
$$\E(S(P_{t+1}) \mid S(P_t),B_j) \leq \E(S(Q_t) \mid S(P_t)) + 2(\mu-1)\cdot 2j=\E(S(Q_t) \mid S(P_t)) + 4(\mu-1)j$$
by Lemma~\ref{lem:diversity-general2} applied with parameters $P_1:=P_{t+1}$ and $P_2:=Q_t$ and summing over the $2j$ distinct bit positions of $y$ and $z$. Since $B_j$ involves flipping $j$ zeros, $q_j$ is bounded by the probability of flipping $j$ many bits. Therefore, we see that the sum $\sum_{j=0}^{\neweditx{k}} jq_j$ is upper bounded by the expected number of bits which are flipped in the whole bit string during standard bit mutation, which is $\chi$. So we obtain
\begin{align*}
\E(S(P_{t+1}) \mid S(P_t)) &= S(P_t)(1-\Pr(V)) + \sum_{j=0}^{\neweditx{k}}\E\neweditx{(}(S(Q_t) \mid S(P_t)) + 4(\mu-1)j)q_j\\
&\leq S(P_t)(1-\Pr(V)) + \E(S(Q_t) \mid S(P_t))\Pr(V) + 4(\mu-1)\chi
\end{align*}
no matter if $E_1$ or $E_2$ occurs. In the following we determine $\E(S(Q_t) \mid S(P_t),E_i)$ for $i \in \{1,2\}$ and start with $i=1$\neweditx{,} which implies that $y=0^{\sqrt{k}}1^{n-k}0^{k-\sqrt{k}}$.
If $x_d=0^{\sqrt{k}}1^{n-k}0^{k-\sqrt{k}}$ then $\E(S(Q_t) \mid S(P_t),E_1) = S(P_t)$. Otherwise, $Q_t$ contains $\mu/4+1$ individuals of the form $0^{\sqrt{k}}1^{n-k}0^{k-\sqrt{k}}$ and $3\mu/4-1$ many of the form $1^{\sqrt{k}}0^{\sqrt{k}}1^{n-k-\sqrt{k}}0^{k-\sqrt{k}}$. \
Since $x_d=0^{\sqrt{k}}1^{n-k}0^{k-\sqrt{k}}$ with probability $1/4$ and two distinct individuals have Hamming distance $2\sqrt{k}$\neweditx{,} which contribute 
$4\sqrt{k}$ to the total diversity (because every pair of different individuals is counted twice), we obtain
\begin{align*}
\E(S(Q_t) \mid S(P_t),E_1) &= \frac{S(P_t)}{4} + \frac{3}{4} \cdot 4\sqrt{k} \cdot (3\mu/4-1)\cdot(\mu/4+1)\\
&= \frac{S(P_t)}{4} + \frac{3S(P_t)}{4}+\frac{3\sqrt{k}(\mu-2)}{2} = S(P_t) + \frac{3\sqrt{k}(\mu-2)}{2}.
\end{align*}
For $i=2$ we obtain in a similar way
\begin{align*}
\E(S(Q_t) \mid S(P_t),E_2) &= \frac{3 S(P_t)}{4} + \frac{1}{4} \cdot 4\sqrt{k} \cdot (3\mu/4+1)\cdot(\mu/4-1)\\
&= \frac{3 S(P_t)}{4} + \frac{S(P_t)}{4}-\frac{\sqrt{k}(\mu+2)}{2} = S(P_t) - \frac{\sqrt{k}(\mu+2)}{2}.
\end{align*}
Suppose that $E_3$ occurs. Denote by $F_{\ell}$ the event that the mutation flips $\ell$ bits. It is well known that 
\[
    \Prob(F_{\ell}) = \binom{n}{\ell}\left(\frac{\chi}{n}\right)^\ell \left(1 - \frac{\chi}{n}\right)^{n-\ell} \le \frac{\chi^\ell}{\ell!}. 
\]
Suppose that we flip $\ell$ bits.  Then we have deterministically $|S(P_{t+1})- S(P_t)| \le 8\sqrt{k}\mu + 2\ell \mu$, which is composed as follows. Removing any search point changes the diversity by at most $4\sqrt{k}\mu$. Adding $y$ changes the diversity by at most $4\sqrt{k}\mu$. And mutating $\ell$ bits in $y$ adds at most $2\ell \mu$. Hence, for $\ell \geq \sqrt{k}$,
\begin{align*}
\sum_{\ell = \lceil\sqrt{k}\rceil}^\infty \Pr(F_\ell) \cdot \E(|S(P_{t+1})-S(P_t)| \mid F_\ell) \le\;& \sum_{\ell = \lceil\sqrt{k}\rceil}^\infty\Pr(F_\ell)\cdot(8\sqrt{k}\mu + 2\ell\mu)\\
\le\;& \sum_{\ell = \lceil\sqrt{k}\rceil}^\infty \frac{\chi^\ell}{\ell!}\cdot(8\sqrt{k}\mu + 2\ell\mu) = o(\sqrt{k} \mu)
\end{align*}
using $k = \omega(1)$.
So assume that the mutation flips $0\le \ell\le \sqrt{k}$ bits. In the following, we condition on $\ell$ and on the positions $b_1,\ldots,b_\ell$ of those bits. Then each of the $2\sqrt{k}$ positions in $z$ still has probability $1/2$ to be $0$ or $1$, independently of each other. (If some of these positions are mutated, this switches the role of $0$ and $1$, but this does not matter due to symmetry, cf.\ Lemma~8 in~\cite{Oliveto2022}.) Hence, if we denote by $N$ the number of one-bits in $z$ in the first $2\sqrt{k}$ positions, then $N$ follows a binomial distribution $N\sim \Bin(2\sqrt{k},1/2)$.
Now, for every value $\ell$ and \neweditx{all} values $b_1,\ldots,b_\ell$ there exists $N_0$ such that the offspring $z$ is in \plateau if and only if $N=N_0$. Therefore, in each such case the probability that the offspring is accepted is $\Pr(N = N_0) \le \max_{j\in \{0, \ldots , 2\sqrt{k}\}}\{\Pr(N = j)\} \le 2k^{-1/4}$ for $k=\omega(1)$. (The maximum grows asymptotically like $(1+o(1))\cdot 2/\sqrt{\pi \sqrt{2k}}$ by the Stirling formula.) Hence, we also have $\Pr(V \mid F_\ell) \le 2k^{\neweditx{-}1/4}$ for every $0\le \ell\le \sqrt{k}$. On the other hand, for every $0\le \ell\le \sqrt{k}$ we have $|S(P_{t+1})- S(P_t)| \le 8\sqrt{k}\mu + 2\ell \mu \le 10\sqrt{k}\mu$. Therefore,
\begin{align*}
\sum_{\ell = 0}^{\lfloor{\sqrt{k}}\rfloor} \Pr(F_\ell) \cdot \E(|S(P_{t+1})-S(P_t)| \mid F_\ell) & \le 10\sqrt{k}\mu\cdot \sum_{\ell = 0}^{\lfloor{\sqrt{k}}\rfloor}\Pr(V \mid F_\ell )\cdot\Pr(F_\ell) \\
& \le 10\sqrt{k}\mu \cdot 2k^{-1/4}\cdot\sum_{\ell = 0}^{\lfloor{\sqrt{k}}\rfloor}\Pr(F_\ell) \\
& \le 20 k^{1/4}\mu.
\end{align*}

\neweditx{Hence, we obtain by the law of total probability
\begin{align*}
\E(S(P_{t+1})- S(P_t) \mid S(P_t),E_3) \leq\;& \sum_{\ell=0}^{\lfloor{\sqrt{k}}\rfloor} \Pr(F_\ell) \cdot \E(|S(P_{t+1})-S(P_t)| \mid F_\ell)\\
&+  \sum_{\ell=\lceil{\sqrt{k}}\rceil}^{\infty} \Pr(F_\ell) \cdot \E(|S(P_{t+1})-S(P_t)| \mid F_\ell)\\
\leq\;& o(\sqrt{k} \mu) + 20k^{1/4} \mu.
\end{align*}
}
Adding up the three cases for $E_1$, $E_2$ and $E_3$, we obtain in total\neweditx{, along with $\Pr(V \mid E_1) = \Pr(V \mid E_2) = (1-\chi/n)^n \geq e^{-\chi}/2$ for $n$ sufficiently large,} 
\begin{align*}
\E(S(P_{t+1}) \mid S(P_t)) - S(P_t) &\leq \frac{1}{16}\cdot\left( \frac{3 \sqrt{k}(\mu-2)}{2} \neweditx{\mathrel{\cdot} \Pr(V \mid E_1)} + 4(\mu-1) \chi\right)\\
& + \frac{9}{16}\cdot\left(-\frac{\sqrt{k}(\mu+2)}{2}\neweditx{\mathrel{\cdot}\Pr(V \mid E_2)} + 4(\mu-1)\chi\right)\\
&+ \frac{6}{16}\left(o(\sqrt{k}\mu)+20k^{1/4}\mu\right)\\
&= - \neweditx{\frac{3 \sqrt{k} \mu e^{-\chi}}{32}} + o(\sqrt{k}\mu),
\end{align*}
which gives a loss in the expected diversity of at least $\neweditx{(1-o(1)) \cdot 3 \sqrt{k} \mu e^{-\chi}/32}$ in total for $k \in \omega(1)$.
\end{proof}

\subsection{Proof of Lemma~\ref{lem:negative-standard-GA-2}}

Finally, we prove Lemma~\ref{lem:negative-standard-GA-2}, which states that a negative drift is possible for the population diversity in the original \muga.
\begin{proof}
    The formula for $S(P_t)$ was already shown in~\eqref{eq:counterexample_population_diversity}. For $E(S(P_{t+1}))$ we first consider a mutation-only step (\neweditx{which} occurs with probability $(1-p_c)$). We imagine a mutation step being split in two steps: At first one individual chosen uniformly at random is cloned and then the respective number of bits in the cloned individual is flipped. Note that cloning can be seen as 
    \emph{boring crossover}~\cite{Friedrich2023} where two parents are chosen (successively) at random and one parent uniformly at random as offspring is taken. By Theorem~26 in~\cite{Lengler2024} this crossover operator is diversity-neutral.
    Hence, by Lemma~\ref{lem:crossover-general} and Lemma~\ref{lem:diversity-general2}, if the offspring is on \plateau,
    \[
    \E(S(P_{t+1}) \mid S(P_t)) \leq \left(1-\frac{2}{\mu^2}\right)S(P_t) +4(\mu-1) \sum_{j=0}^k j q_j
    \]
    where $q_j$ denotes the probability of flipping $j$ ones and $j$ zeros. 
    Obviously, $q_j$ is upper-bounded by the probability of flipping exactly $j$ zeros and an arbitrary number of ones. Hence, the sum $\sum_{j=0}^kj q_j$ is upper-bounded by the expected number of zero-bits that are flipped, which is $k\chi/n$. Therefore, 
    \begin{align}
    \label{eq:plateau-and-not}
    \E(S(P_{t+1}) \mid S(P_t)) &\leq \left(1-\frac{2}{\mu^2}\right)S(P_t) +\frac{4(\mu-1) k\chi}{n} \leq S(P_t) +\frac{4\mu k \chi}{n}
    \end{align}
    for a mutation step. Note that inequality~\eqref{eq:plateau-and-not} is trivially fulfilled if the offspring is not on \plateau. Together with Lemma~\ref{lem:negative-standard-GA} and the law of total probability, we obtain the result.
\end{proof}

\end{document}